\documentclass[11pt]{article}

\usepackage[round]{natbib}

\usepackage{url}            % simple URL typesetting
\usepackage{booktabs}       % professional-quality tables
\usepackage{amsfonts}       % blackboard math symbols
\usepackage{nicefrac}       % compact symbols for 1/2, etc.
\usepackage{microtype}      % microtypography
\usepackage{xcolor}         % colors

\usepackage{amsmath,mathtools}
\usepackage{enumitem} 
\usepackage{amsthm}         %Thm, proof 
\usepackage{caption}        %subfigure
\usepackage{subcaption}     %subfigure
\usepackage[linesnumbered,ruled,vlined]{algorithm2e} % algorithm
\SetCommentSty{mycommfont}
\SetKwInput{KwInput}{Input}                % Set the Input
\SetKwInput{KwOutput}{Output}              % set the Output
\usepackage{xr-hyper}                   % for appendix cross reference
\usepackage{hyperref}       % hyperlinks
\usepackage{comment}        % begin comment
\usepackage{multirow}

\newtheorem{thm}{Theorem}
\newtheorem*{theorem*}{Theorem}

\newtheorem{lem}[thm]{Lemma}
\newtheorem{prop}[thm]{Proposition}
\newtheorem{defn}{Definition}
\newcommand{\eps}{\varepsilon}

\def\argmin{\mathop{\rm argmin}}

\def\Var{\mbox{Var}}

\def\E{\mbox{E}}

\def\diag{\mbox{diag}}

\def\diag{\mbox{Diag}}

\def\av{\mathbf a}
\def\bv{\mathbf b}

\def\ev{\mathbf e}

\def\sv{\mathbf s}
\def\tv{\mathbf t}
\def\uv{\mathbf u}
\def\vv{\mathbf v}
\def\wv{\mathbf w}
\def\xv{\mathbf x}
\def\yv{\mathbf y}
\def\zv{\mathbf z}

\def\Av{\mathbf A}

\def\Dv{\mathbf D}
\def\Ev{\mathbf E}

\def\Iv{\mathbf I}

\def\Lv{\mathbf L}

\def\Sv{\mathbf S}

\def\Uv{\mathbf U}
\def\Vv{\mathbf V}
\def\Wv{\mathbf W}
\def\Xv{\mathbf X}
\def\Yv{\mathbf Y}
\def\Zv{\mathbf Z}

\newcommand{\etav}{\mbox{\boldmath{$\eta$}}}

\newcommand{\Piv}{\mbox{\boldmath{$\Pi$}}}

\newcommand{\Sigmav}{\mbox{\boldmath{$\Sigma$}}}
\newcommand{\Lambdav}{\mbox{\boldmath{$\Lambda$}}}

\newcommand{\Bc}{\mathcal{B}}

\newcommand{\Fc}{\mathcal{F}}

\newcommand{\Pc}{\mathcal{P}}

\newcommand{\Sc}{\mathcal{S}}

\newcommand{\Uc}{\mathcal{U}}
\newcommand{\Vc}{\mathcal{V}}
\newcommand{\Wc}{\mathcal{W}}
\newcommand{\Xc}{\mathcal{X}}

\newcommand{\Rb}{\mathbb{R}}

\newcommand{\gr}{{\rm{Gr}}}
\newcommand{\sph}{{\rm{Sp}}}
\newcommand{\bp}{{\rm{bp}}}
\newcommand{\row}{{\rm{row}}}
\newcommand{\col}{{\rm{col}}}
\newcommand{\BP}{{\rm{BP}}}

\def\1v{\mathbf 1}
\def\0v{\mathbf 0}

\title{Robust SVD Made Easy: A fast and reliable algorithm for large-scale data analysis}
\author{Sangil Han, Kyoowon Kim, and Sungkyu Jung\\
Department of Statistics Seoul National University}
\begin{document}
\maketitle

\begin{abstract}
  The singular value decomposition (SVD) is a  crucial tool in machine learning and statistical data analysis. However, it is highly susceptible to outliers in the data matrix. Existing robust SVD algorithms often sacrifice speed for robustness or fail in the presence of only a few outliers. This study introduces an efficient algorithm, called Spherically Normalized SVD, for robust SVD approximation that is highly insensitive to outliers, computationally scalable, and provides accurate approximations of singular vectors. The proposed algorithm achieves remarkable speed by utilizing only two applications of a standard reduced-rank SVD algorithm to appropriately scaled data, significantly outperforming competing algorithms in computation times. To assess the robustness of the approximated singular vectors and their subspaces against data contamination, we introduce new notions of breakdown points for matrix-valued input, including row-wise, column-wise, and block-wise breakdown points. Theoretical and empirical analyses demonstrate that our algorithm exhibits higher breakdown points compared to standard SVD and its modifications. We empirically validate the effectiveness of our approach in applications such as robust low-rank approximation and robust principal component analysis of high-dimensional microarray datasets. Overall, our study presents a highly efficient and robust solution for SVD approximation that overcomes the limitations of existing algorithms in the presence of outliers.
\end{abstract}

\section{INTRODUCTION}\label{sec:intro}

Singular Value Decomposition (SVD) is one of the most useful tools in machine learning, used in processing image, video and natural languages, constructing recommender systems, and statistical data analysis. In particular, SVD is used for dimension reduction for downstream machine learning tasks, which often improves the overall performance of the task with reduced computational complexity. However, real-world data often contain noise, outliers, and other anomalies, and with only a contaminated data matrix at hand, standard SVD may provide undesirable low-rank decomposition. 
As a result, there is a need for robust SVD algorithms that can handle these challenges and provide accurate and robust results in the presence of data irregularities.

In this work, we introduce a new approach called Spherically Normalized SVD (SpSVD), which aims to handle outliers more effectively compared to the classical SVD and competing robust SVD algorithms including \cite{zhang2013robust,candes2011robust,brahma2017reinforced,rahmani2017coherence}. 
Inspired by a robust PCA proposal of \citet{locantore1999robust}, the SpSVD algorithm adopts the spherical normalization approach of Locantore et al. to approximate both left and right singular vectors. While the normalization gives highly robust approximations of those vectors, we additionally solve a simple optimization problem  for more accurate low-rank approximation. Our algorithm is easy to implement and extremely fast to compute. Specifically, it requires a computational complexity similar to that of classical SVD for low-rank approximations. We also establish that the algorithm provides a statistically accurate approximation of singular vectors, even in the presence of infinitesimal contamination of a considerable scale.

The robustness of SpSVD is carefully evaluated by extending the notion of breakdown point, a quantitative measure of robustness, commonly used in robust estimation literature such as \cite{tyler2023robust,tang2016robustness,huber2011robust}. The classical definition of observation-wise breakdown point is generalized to handle singular vectors (unit vectors) and the subspaces spanned by those and also for the cases where contamination occurs for rows, columns or individual elements of input matrix $\Xv$. Utilizing the generalized notion of breakdown point, we show that the singular vectors approximated by SpSVD have higher breakdown points than the classical SVD and its seemingly robust  variants \citep{gabriel1979lower,liu2003robust,ke2005robust,zhang2013robust}. To the best of our knowledge, this work is the first study investigating breakdown points specifically related to the singular vectors and their subspaces, providing novel insights into the robustness of SVD.

The accuracy, robustness, and computational times of SpSVD are empirically compared with existing robust SVD approaches including \cite{zhang2013robust,candes2011robust,brahma2017reinforced,rahmani2017coherence}, via simulated experiments. In particular, our proposal is on par with the best-performing algorithm---the Robust PCA (RPCA) approach of \citet{candes2011robust}---in terms of the accuracy and  robustness, but is up to 500 times faster (in real computation times) than RPCA, rendering its effectiveness especially for large-scale data analysis.

\paragraph{Related Works}

There have been many proposals for robust SVD or PCA (Principal Component Analysis), which can be broadly categorized into four approaches. 
{(1) Optimization for low-rank approximation}: One approach to low-rank approximation is achieved by solving optimization problems \citep{markopoulos2014optimal,barrodale1968l1,ding2019noisy}. Specifically, \citet{candes2011robust} proposed an optimization problem where a data matrix is decomposed into a low-rank approximation and a sparse outlier matrix. This approach has many variants \citep{wright2013compressive,zhou2010stable,xu2010robust}. One of them, \citet{she2016robust,brahma2017reinforced} proposed using the orthogonal complement of the low-rank approximation to capture outliers lying the orthogonal complement space. 
{(2) Minimization of element-wise loss}: Finding SVD can be recasted to a minimization problem with element-wise loss. To modify SVD to be more robust, several approaches including \cite{gabriel1979lower,liu2003robust,ke2005robust,zhang2013robust} have been proposed to replace the element-wise loss with other loss functions such as $L_1$-loss or Huber's loss function. 
{(3) Outlier filtering}: In a natural attempt to achieve robustness, many researchers have proposed outlier filtering methods \citep{xu2012outlier, diakonikolas2023nearly, rahmani2017coherence, kong2020robust,jambulapati2020robust} including robust covariance estimation by subsampling techniques. 
However, most robust covariance matrix estimations including \cite{campbell1980robust,rousseeuw1984least,rousseeuw1985multivariate,hubert2018minimum} require large sample size $(n \gg p)$, and are not generally not applicable to SVD. 
{(4) Projection pursuits}: In that the first principal component is the direction maximizing the variability of data, researchers including \cite{croux2013robust,croux2007algorithms} proposed approaches based on projection pursuit where the goal is to find the direction that maximizes the dispersity of data points for some robust dispersity measure, such as the first quantile of the pairwise differences. Most of these approaches are not applicable to large-scale data analysis, due to its high computational costs. We have numerically compared our proposal with most of the approaches above, but chose not to present the results because of either poor performances or excessively long computation times. 

Most of the aformentioned algorithms do not provide robust SVD, but robust PCA, which primarily concentrate on subspace recovery. Such methods may robustly recover right singular vectors, but do not provide robust left singular vectors and singular values. Note that there are data-analytic situations in which both the right and left singular vectors are simultaneously needed, for instance in multi-source data analysis \citep{feng2018angle,lock2013joint,prothero2022data}.

\citet{hampel1968contributions}  and \citet{huber1992robust} introduced the notion of breakdown point, a  quantitative measure of robustness in the presence of distributional contamination. They also introduced a probabilistic measure of robustness via the influence function. 
Some researchers have investigated robustness of unit vectors via influence functions \cite{ko1988robustness, ko1993robust}. Our work is the first to extend the notion of breakdown points for functions taking values in the unit sphere and Grassmannian manifolds.

%--------------------------------------------------------------------------------------------
\section{PROPOSED ALGORITHM: Spherically Normalized SVD}\label{sec:SpSVD}

\paragraph{Background}

The SVD of a $n\times p$ real-valued matrix $\Xv$ is denoted by $\Xv = \Uv \Dv \Vv^T = \sum_{r=1}^{n \land p} d_r \uv_r \vv_r^T$, where the diagonal matrix $\Dv$ contains the non-negative singular values $d_r$, arranged in descending order, and $\uv_r$ and $\vv_r$ are the $r$th left and right singular vectors, respectively, corresponding to the $r$th largest singular value $d_r$.  
To motivate our construction of robust SVD algorithms, we view the real-valued matrix $\Xv = (\xv_1,\dots,\xv_n)^T \in \Rb^{n\times p}$  as a data matrix, collecting the observed values of $p$ variables from $n$ individuals. 

The principal component analysis (PCA) applied to the data matrix $\Xv$ is closely related to the SVD of $\Xv$. Assuming that $\Xv$ is column-centered, the $r$th (empirical) principal component (PC) direction vector is the $r$th right singular vector $\vv_r$. The $r$th PC scores are given by the $n$-vector $\Xv \vv_r = (\xv_1^T\vv_r,\ldots,\xv_n^T\vv_r)^T$, consisting of the projection of each data point onto $\vv_r$. In fact, the PC score vector is also given by the SVD, that is, $\Xv \vv_r = d_r \uv_r$, in which the unit vector $\uv_r$ collects the standardized PC scores for $n$ individuals, and the sample standard deviation of the PC scores is $d_r / \sqrt{n}$. Dimensionality reduction in data matrix $\Xv$ can be equivalently achieved either by a rank-$R$ SVD approximation $\widehat\Xv_{R}^{\rm svd} := \sum_{r=1}^{R} d_r \uv_r \vv_r^T$, or by collecting the first $R$ triples of PC direction, score and standard deviation. Note that  $\widehat\Xv_{R}^{\rm svd}$ is the best rank-$R$ approximation to $\Xv$ in terms of the Frobenius norm. 

\paragraph{Motivation}

Both the SVD and standard PCA estimates are highly sensitive to data contamination in the data matrix $\Xv$. Consider an original data matrix $\Xv$ and a contaminated version $\Zv$. In the contaminated version, the values of the first row $\xv_1^T$ of $\Xv$ are replaced with arbitrary values. Denoting $\vv_r(\Xv)$ and $\vv_r(\Zv)$ as the $r$th right singular vectors of $\Xv$ and $\Zv$ respectively, we demonstrate in Section~\ref{sec:robust} that the singular vectors are sensitive to their input. Even if there is only one differing row (data point) between $\Xv$ and $\Zv$, the difference between $\vv_1(\Xv)$ and $\vv_1(\Zv)$ can be substantial. This occurs due to the presence of arbitrarily large $\| \zv_1 \|_2$ values, causing the singular vector $\vv_1(\Zv)$ to align almost parallel to $\zv_1$.
To limit such a potentially massive contribution of a single observation, \citet{locantore1999robust} proposed to \emph{normalize} each of $n$ data points in $\Xv$, transforming $\xv_i$ to $\xv_i / \|\xv_i\|_2$, then to apply the standard PCA algorithm for the normalized data, for their proposal of a robust PCA. Since all data points are on the unit sphere in $\Rb^p$ after normalization, the contribution of potential outliers is naturally limited.  
Building upon the robust PCA approach of \cite{locantore1999robust}, we propose a novel technique called \emph{Spherically Normalized SVD algorithm} (SpSVD for short), which provides a robust approximation of the first $R$ left and right singular vectors and singular values of (uncontaminated) $\Xv$, obtained purely from potentially contaminated matrix $\Zv$.

\paragraph{Algorithm}

Let $\Xv = [\xv_1,\dots,\xv_n]^T \in \Rb^{n \times p}$ be a potentially contaminated data matrix. For a predetermined  rank $R \in \{1,\dots, n \land p\}$ our goal is to define the ordered triple $(d_r^{\sph},\uv_r^{\sph},\vv_r^{\sph})$, for $r = 1,\dots,R$, that provides a highly insensitive and accurate rank-$R$ approximation $\widehat\Xv_{R}^{\sph} := \sum_{r=1}^{R} d_r^{\sph} \uv_r^{\sph} (\vv_r^{\sph})^T$ of $\Xv$. 

For the approximation of right singular vectors, we individually scale each row of matrix $\Xv$ to have unit length. Let $\widetilde{\Xv}_{\rm row} = [\xv_1 / \|\xv_1\|_2,\dots,\xv_n / \|\xv_n\|_2]^T$ represent the row-normalized data matrix. Subsequently, a standard low-rank SVD algorithm is applied to $\widetilde{\Xv}_{\rm row}$ to obtain the $R$ right singular vectors corresponding to the $R$ largest singular values of $\widetilde{\Xv}_{\rm row}$. The set of these right singular vectors is denoted as $V^R = 
\{\vv_1(\widetilde{\Xv}_{\rm row}), \ldots, \vv_R(\widetilde{\Xv}_{\rm row})\}$. Similarly, for the approximation of left singular vectors, we scale and collect each \emph{column} of $\Xv$ in the column-normalized data matrix $\widetilde{\Xv}^{\rm col}$. The first $R$ left singular vectors of $\widetilde{\Xv}^\col$ are then collected in the set $U^R = \{\uv_1(\widetilde{\Xv}^{\rm col}), \ldots, \uv_R(\widetilde{\Xv}^{\rm col})\}$. While the elements in $V^R$ and $U^R$ are candidates for  $(\vv_r^{\sph}$ and $\uv_r^{\sph})$, respectively, we do not set $(\uv_r(\widetilde{\Xv}^{\rm col}), \vv_r(\widetilde{\Xv}_{\rm row}))$ for $(\uv_r^{\sph},\vv_r^{\sph})$. This is because using mismatched labels  can provide a better approximation of $\Xv$. 

With $V^R$ and $U^R$ at hand, the triple $(d_r^{\sph},\uv_r^{\sph},\vv_r^{\sph})$ is defined sequentially. 
For the first triple, we solve the following: 
\begin{align}\label{eq:spsvd_1}
    (d_1^{\sph},\uv_1^{\sph},\vv_1^{\sph}) = \argmin_{ d \in \Rb, \uv \in U^R, \vv \in V^R} 
    \|\Xv - d\uv\vv^T\|_{\rm{F}_1},
\end{align}
where $\|\Av\|_{\rm{F}_1} = \sum_{i,j} |a_{ij}|$ is the element-wise 1-norm of the matrix $\Av$. For a fixed pair of $\uv=(u_1,\dots,u_n)^T \in U^R$ and $\vv = (v_1,\dots,v_p)^T \in V^R$, finding the solution to \eqref{eq:spsvd_1} with respect to $d\in \Rb$ is equivalent to solving a weighted median problem,\footnote{%which
The weighted median problem can be efficiently solved by the median of medians algorithm \citep{blum1973time}, or by other selection algorithms \citep{cormen2022introduction}. Therefore, by solving the weighted median for all candidate pairs, we obtain a solution to \eqref{eq:spsvd_1}.}
\begin{align}\label{eq:weighted_median}
    \min_{d} \sum_{u_{i}\neq 0, v_{j} \neq 0}|{u_{i}v_{j}}||\frac{x_{ij}}{u_{i}v_{j}} - d|.
\end{align}
For $r = 2,\ldots, R$, optimization problems similar to \eqref{eq:spsvd_1} are used to define the $r$th triple $(d_r^{\sph},\uv_r^{\sph},\vv_r^{\sph})$, but with the first $r-1$ triples deflated from each of $\Xv$, $\Uv_R$ and $\Vv_R$. That is, we solve  
\begin{align}\label{eq:spsvd}
    (d_r^{\sph},\uv_r^{\sph},\vv_r^{\sph}) = \argmin_{ d \in \Rb, \uv \in U_{r}^{R}, \vv \in  V_{r}^{R}} 
    \|\Xv_r - d\uv\vv^T\|_{\rm{F}_1},
\end{align} 
where $\Xv_r = \Xv - \sum_{l=1}^{r-1}d_l^{\sph}\uv_{l}^{\sph}(\vv_{l}^{\sph})^T$, 
$U_{r}^R = U^R \setminus \{\uv_1^{\sph},\ldots, \uv_{r-1}^{\sph}  \},$ and $V_{r}^R = V^R \setminus \{\vv_1^{\sph},\ldots, \vv_{r-1}^{\sph} \}$. Note that in \eqref{eq:spsvd_1} and 
 \eqref{eq:spsvd} above, the approximated singular value $d_r^{\sph}$ may be negative. Since singular values are, by definition, non-negative,  for each $r = 1,\ldots, R$, we update $(d_r^{\sph},\uv_r^{\sph},\vv_r^{\sph})$ by 
 $( s_r d_r^{\sph},s_r \uv_r^{\sph},\vv_r^{\sph})$, if $s_r := {\rm sign}(d_r^{\sph}) \neq 0$. This procedure is summarized in Algorithm~\ref{alg:approx}. 

\setlength{\algomargin}{1.5em}
\begin{algorithm}[t]
\DontPrintSemicolon
  \KwInput{data matrix $\Xv = (\xv_1,\dots,\xv_n)^T = (\xv^1,\dots,\xv^p) \in \Rb^{n \times p}, R \in \{1,\dots, n \land p\}.$}
  \KwOutput{$(d_r^{\sph},\uv_r^{\sph},\vv_r^{\sph})$ for $r=1,\dots,R$}
  Normalize $\widetilde{\Xv}_{\row} = [\xv_1 / \|\xv_1\|_2,\dots,\xv_n / \|\xv_n\|_2]^T$ and $\widetilde{\Xv}^{\col} = [\xv^1 / \|\xv^1\|_2,\dots,\xv^p / \|\xv^p\|_2]$ \;
  Apply SVD to $\widetilde{\Xv}_{\rm{row}}$ to obtain rank-$R$ approximation $\sum_{r=1}^{R}\widetilde{d}_r\widetilde{\uv}_r \left( {\vv}_r(\widetilde{\Xv}_{\row}) \right)^T$ \;
  Apply SVD to $\widetilde{\Xv}^{\rm{col}}$ to obtain rank-$R$ approximation $ \sum_{r=1}^{R}\widetilde{c}_r \left(\uv_r(\widetilde{\Xv}^{\rm col})\right) \widetilde{\vv}_r^T$  \;
  Set $V^R =  \{\vv_1(\widetilde{\Xv}_{\rm row}), \ldots, \vv_R(\widetilde{\Xv}_{\rm row})\}$ and $U^R = \{\uv_1(\widetilde{\Xv}^{\rm col}), \ldots, \uv_R(\widetilde{\Xv}^{\rm col})\}$\;
  \For{$r=1,\dots,R$}{
  Find the $r$th solution $(d_r^{\sph},\uv_r^{\sph},\vv_r^{\sph})$ of \eqref{eq:spsvd}\;
  Update $(d_r^{\sph},\uv_r^{\sph},\vv_r^{\sph})$ by 
 $( s_r d_r^{\sph},s_r \uv_r^{\sph},\vv_r^{\sph})$, if $s_r := {\rm sign}(d_r^{\sph}) \neq 0$\; 
  }
\caption{Rank-$R$ Approximation by SpSVD}\label{alg:approx}
\end{algorithm}

 We remark that one may try to choose $(\uv_r^\sph, \vv_r^\sph)$
 to be 
 \begin{equation}
     \label{eq:naive_choice}
     (\uv_r^\sph, \vv_r^\sph) := (\uv_r(\widetilde{\Xv}^{\rm col}),\vv_r(\widetilde{\Xv}_{\rm row})),
 \end{equation} 
 and solve \eqref{eq:spsvd_1} and 
 \eqref{eq:spsvd} only with respect to $d$. 
 Since the candidates in $V^R$ and $U^R$ are from different normalizations, there is no compelling reason to believe that the orders in the left and right singular values are related as in \eqref{eq:naive_choice}. We have empirically found that the solutions of \eqref{eq:spsvd_1} and 
 \eqref{eq:spsvd} 
 are typically different from the naive choice \eqref{eq:naive_choice}, more often when $R$ is large. Our approach of searching over the $R^2$ pairs of candidates provides generally better rank-$R$ approximations than the naive choice. On the other hand, potential downsides of our approach include the following: The $r$th triple $(d_r^{\sph},\uv_r^{\sph},\vv_r^{\sph})$ may depend on the choice of the rank $R \geq r$, and the computational complexity increases because of the discrete optimization needed in \eqref{eq:spsvd}. Nevertheless, for small $R$ the overall computational complexity of the proposed SpSVD algorithm is similar to that of a standard low-rank SVD algorithm.

 \paragraph{Computational Complexity}
 
The computational complexity of an algorithm is a measure of the amount of resources required for solving a problem of a given size. The overall computational complexity of Algorithm~\ref{alg:approx} applied to rank-$R$ approximation of a real matrix of size $n\times p$ is $O(npR^3)$. While the two applications of SVD as well as the normalization require $O(npR)$, the algorithm solves the weighted median problem \eqref{eq:weighted_median} in the linear time complexity of $O(np)$ for $R(R+1)(2R+1)/3$ times. As a comparison, the computational complexity of the standard rank-$R$ SVD algorithm is $O(npR)$ \citep{xu2023fast, yi2016fast, shamir2016fast, allen2016lazysvd}. Algorithm~\ref{alg:approx} is as efficient as SVD when $R$ is small.  As we compare numerically in Section~\ref{sec:numeric}, computing the approximation of a low-rank SVD  by our approach is up to 500 times (on average in real computation times) faster than state-of-the-art robust SVD algorithms proposed in \cite{zhang2013robust,brahma2017reinforced,candes2011robust}.
 
\paragraph{Statistical Accuracy}

When the data matrix $\Xv \in \Rb^{n\times p}$ is viewed as a  collection of $n$ observations $\xv_i$, the right singular vector $\vv_r(\Xv)$ is equivalent to the eigenvector of the sample covariance matrix of $\Xv$. Treating $\vv_r^\sph$ as an estimator of the $r$th eigenvector of the population covariance matrix, the estimator $\vv_r^\sph$ is consistent under adequate assumptions; see Section~\ref{subapp:consistency} of the supplementary material. 
While the standard SVD also provides a consistent estimator, SpSVD exhibit a statistical accuracy over contamination. Let $\Fc_{\Sigmav}$ denote a mean-zero, $p$-dimensional elliptical distribution \citep{cambanis1981theory}, with covariance matrix $\Sigmav$, whose eigen-decomposition is given by $\Sigmav = \sum_{j=1}^{p} \lambda_{j} \vv_{j} \vv_{j}^T$. Assume that $\lambda_{j+1} > \lambda_{j}$ for all $j$. Let $\epsilon > 0$ be the fraction of contaminations among $n$ samples: That is, the samples consist of both i.i.d (uncontaminated) samples $\xv_1, \dots, \xv_{(1-\epsilon)n} \sim \Fc_{\Sigmav}$ and the outliers $\yv_{1}, \dots, \yv_{\epsilon n}$ of arbitrary size and directions. Let ${\vv}_{j}^{\sph}$ be the $j$th left singular vector obtained by the proposed SpSVD applied to $[\xv_{1} , \dots, \xv_{(1-\epsilon)n},\yv_{1} , \dots, \yv_{\epsilon n}]^T$.

\begin{thm}[Statistical accuracy over infinitesimal contamination]{\label{thm:finsamp}}
     If (i) $n \ge Cp/\epsilon$, for a constant $C>0$, and (ii) for $j =1, \dots, p, \delta_j := \lbrace \lvert d_{j+1} - d_{j} \rvert, \lvert d_{j} - d_{j-1} \rvert \rbrace >0$ where $d_{j}$ is the $j$th singular value of the covariance matrix of $\frac{\xv_{1}}{\lVert \xv_{1} \rVert_{2}}$, then
    \begin{align}\label{eq:acc}
        \mathbb{E} [\min \lbrace \lVert {\vv}_{j}^{\sph} &- \vv_{j} \rVert_{2} , \lVert {\vv}_{j}^{\sph} + \vv_{j} \rVert_{2} \rbrace ] \nonumber \\ &\leq \frac{1}{\delta_j}(C' \epsilon + C'' \sqrt{(1-\epsilon)\epsilon})
    \end{align}
    for some absolute constants $C', C'' > 0$.
\end{thm}
In the above theorem, as the fractions $\epsilon$ and $\frac{p}{n}$ go to zero, the singular vectors of our method converge to the target singular vectors in an appropriate statistical context. The assumption $(i)\ n \geq Cp / \epsilon$ is more relaxed than the assumption $``n \geq  Cp / \epsilon^2"$ provided in \cite{diakonikolas2023nearly} (Similar assumptions were imposed in \cite{jambulapati2020robust,kong2020robust}.) If we assume $n \geq Cp / \epsilon^2$, the second term of \eqref{eq:acc} becomes negligible. Our result states the accuracy of unit singular vectors $\vv_j^\sph$ directly compared to $\vv_j$ (for all $j=1,\dots,p$). This is in contrast to the statements in \cite{xu2012outlier,jambulapati2020robust,diakonikolas2023nearly}, in which the accuracy of the only the first vector $\widehat{\vv}_1$ is compared indirectly by bounding $|\widehat{\vv}^T\Sigmav\widehat{\vv} - \lambda_1|$.

%--------------------------------------------------------------------------------------------
\section{EXTENSIONS OF BREAKDOWN POINTS}\label{sec:robust}
 
The breakdown point, originally proposed by \cite{hampel1968contributions}, and studied by \cite{huber1983notion,huber1984finite,huber2011robust}, is a common tool for evaluating quantitative robustness of statistics. Viewing a real-valued statistic as a function $f:\Xc^n \rightarrow \Rb$ that takes as input $n$ data points $\Xv := (\xv_1,\ldots,\xv_n) \in \Xc^n$ and outputs a real-valued $f(\Xv)$, the breakdown point of $f$ at the given data $\Xv$ is defined as the minimum number of corrupted data points that cause the statistic to ``break down.''
Formally, the breakdown point of $f$ at $\Xv$ is
\begin{align}\label{eq:bd_scalar}
\begin{split}
    \bp(f;\Xv) :=  \min_{1\leq l \leq n} \{l : \sup_{\Zv_l} |f(\Zv_l)-f(\Xv)|=\infty\},
\end{split}
\end{align}
where the supremum is taken over all possible corrupted collections $\Zv_l$ that are obtained from $\Xv$ by replacing $l$ data points of $\Xv$ with arbitrary values. 
In this sense, a function $f$ (giving the value of a statistic $f(\Xv)$) is said to \emph{break down} if the difference between the statistics computed  from corrupted data and from the original data, i.e., $f(\Zv_l)$ and $f(\Xv)$, is as large as possible, which was defined to be the infinity in the original definition \eqref{eq:bd_scalar} of the breakdown point. 
Note that in the literature \citep{huber1983notion,huber1984finite,huber2011robust,lopuhaa1991breakdown}, the (finite-sample) breakdown point is in fact defined as $\bp(f;\Xv) / n$, the fraction of the number of corrupted data points and the sample size. Nevertheless, for notational simplicity, we regard the number of data points $n$ as fixed, and define the breakdown point as a whole number.  
The breakdown point $\bp(f;\Xv)$ represents a critical threshold where a breakdown of $f$ does not occur when the number of corrupted data points is below $\bp(f;\Xv)$, but breakdown can occur when the number of corrupted data points is equal to or exceeds $\bp(f;\Xv)$.

We extend the notion of breakdown points to the situations where the singular vectors $\uv_r(\Xv)$ and $\vv_r(\Xv)$ and the subspaces spanned by these are the statistics of interest. This involves two distinct considerations: Breakdown of unit-sphere $S^{k-1}$ and grassmannian $\gr(k,r)$-valued statistics, and breakdown with respect to contamination matrices.

\paragraph{Breakdown of Unit Vectors and Subspaces}\label{sec:bd_subspace} 

Viewing each of the singular vectors $\uv_r(\Xv)$ and $\vv_r(\Xv)$ (or $\uv_r^\sph(\Xv)$ and $\vv_r^\sph(\Xv)$) as a statistic, the notion of breakdown point \eqref{eq:bd_scalar} naturally applies with the following modifications. Since a singular vector lies in the unit sphere $S^{k-1}$ (for $k = n$ or $p$), we measure the difference between $\vv, \wv \in S^{k-1}$ by $\theta(\vv, \wv) := \arccos(|\vv^T\wv|)$, the ``angle" between two directions. The maximum difference in this case is $\pi/2$. We will also be interested in the subspaces spanned by singular vectors. Let $\mathcal{V} = \mathcal{V}(\Xv)$ be the subspace spanned by $\{ \vv_r(\Xv): r=  1,\ldots,R\}$. Then 
$\Vc \in \gr(k,R)$, the Grassmannian manifold consisting of $R$-dimensional subspaces in $\Rb^k$. 
For two subspaces $\Vc = {\rm span}(\Vv), \Wc = {\rm span}(\Wv) \in \gr(k,R)$, the difference may be measured via the largest canonical angle $\theta(\Vc, \Wc) := \arccos( d_{\min}(\Vv^T\Wv) )$, where $d_{\min}(\Av)$ is the smallest singular value of $\Av$. Note that for $R = 1$, $\theta({\rm span}(\vv), {\rm span}(\wv)) = \theta (\vv, \wv)$. The maximum difference is also $\pi/2$, and 
we say $\Vc: \Rb^{n\times p} \to \gr(k,R)$ (or $S^{k}$) breaks down at $\Xv$ by replacing $l$ data points, if $\sup_{\Zv_l} \theta (\Vc(\Zv_l), \Vc(\Xv)) = \pi/2$. 
The breakdown point of $\Vc$ at $\Xv$ is then given by \eqref{eq:bd_scalar} with the definition of ``breakdown" given above.

\paragraph{Breakdown with respect to Contamination of Matrices}\label{sec:bd_matrix}

A data matrix, to which SVD is performed, is not always a statistical data matrix in $\Rb^{n \times p}$ consisting of $n$ data points with $p$ variables. We generalize the mechanism of data contamination from the observation-wise data contamination \eqref{eq:bd_scalar} to three different types of data contaminations. Treating each row, column, or element as a data point, we will discuss row-wise, column-wise, and block-wise contaminations of data matrix $\Xv$.

Let $\Vc: \Rb^{n \times p} \rightarrow \gr(k,r)$ be a statistic of interest. The row-wise breakdown point measures the robustness of $\Vc$ in terms of contamination of the rows of input matrix $\Xv$, and coincides with the traditional notion of breakdown point \eqref{eq:bd_scalar} when rows represent data points. That is, we define $\bp_{\row}(\Vc ;\Xv) = \min\{l : \sup_{\Zv_l} \theta(\Vc(\Zv_l),\Vc(\Xv)) ={\pi}/{2}, 1\leq l \leq n \}$, where the supremum is taken over all possible $\Zv_l$ obtained by replacing $l$ rows of $\Xv$ by arbitrary values. Similarly, the column-wise breakdown point measures the robustness of $\Vc$ with respect to contaminated columns (viewing each column as a data point), and is $\bp_{\col}(\Vc; \Xv) = \min\{l : \sup_{\Zv^l} \theta \big(\Vc(\Zv^l),\Vc(\Xv)\big) ={\pi}/{2}, 1\leq l \leq p \}$, in which ${\Zv^l}$ is given by replacing $l$ columns of $\Xv$.

In situations where each element of the $n \times p$ matrix $\Xv$ is considered as an observation, outlying observations may be scattered across the matrix, and it becomes challenging to devise an informative notion of robustness. We focus on the case that contamination occurs within a (possibly non-consecutive) block. For example, if three elements $x_{1,1}$, $x_{1,3}$ and $x_{2,3}$ of $\Xv$ are contaminated, then we say the outliers lie in a block of size $(2,2)$, in which the numbers correspond to two rows and two columns, respectively. We say $\Vc$ breaks downs at block-size  $(k,l)$ (at $\Xv$) if $\sup_{\Zv_k^l} \theta(\Vc(\Zv_{k}^l),\Vc(\Xv)) ={\pi}/{2}$, 
    where the corrupted data $\Zv_k^l$ are given by replacing the elements in a $k\times l$ block of $\Xv$.
Recall that the breakdown point $\bp(f; \Xv)$ \eqref{eq:bd_scalar} is a threshold, i.e., the minimum number of data points needed for the statistic to break down. To extend the definition of breakdown point to this block-wise contamination scenario, we adopt a partial order relation ``$\prec$'' among the block-sizes in $\Bc_{n,p} = \{(k,l) : 1 \leq k \leq n, 1 \leq l \leq p \}$, given by 
(i) $(i,j) \preceq (k,l)$ if $i \leq k$ and $j \leq l$, and (ii) $(i,j) \prec (k,l)$ if $(i,j) \preceq (k,l)$ and $(i,j) \neq (k,l)$.

The set $(\Bc_{n,p}, \prec)$ is only partially ordered, meaning that there are block-sizes $(i,j)$ and $(k,l)$ that can not be ordered; take $(2,3)$ and $(3,1)$ as an example. This is unavoidable due to the two-dimensional nature of $\Bc_{n,p}$. Nevertheless, utilizing the partial order provides a definition for block-wise breakdown point as a ``tipping'' point. 

\begin{defn}\label{def:bd_block}
 We say that $\Vc$ has a block-wise breakdown point $(i,j)$  at $\Xv$ if
        (i) for any $(k,l) \succeq (i,j)$, $\Vc$ breaks down at block-size $(k,l)$, and
         (ii) for any $(k',l') \prec (i,j)$, $\Vc$ does not break down at block-size $(k',l')$. 
\end{defn}
It is possible that there are multiple block-wise 
 breakdown points for $\Vc$, and we denote the set of all block-wise breakdown points of $\Vc$ at $\Xv$ by $\BP(\Vc;\Xv)$. 
The notion of block-wise breakdown points is more powerful than row and column-wise breakdown points, as the following lemma states. 
\begin{lem}\label{lem:bd_relation}\ 
        (i) If $(k,1) \in \BP(\Vc;\Xv)$ for some $1 \leq k \leq n$, then $\bp_{\col}(\Vc;\Xv) =1$, and 
        (ii) if $(1,l) \in \BP(\Vc;\Xv)$ for some $1 \leq l \leq p$, then $\bp_{\row}(\Vc;\Xv) =1$.
\end{lem}

%--------------------------------------------------------------------------------------------
\section{ROBUSTNESS OF SpSVD} \label{sec:blockwise}

 We begin by highlighting the lack of robustness of the standard SVD.  
Let 
 ${\Vc}_R:\Rb^{n\times p} \rightarrow {\gr}(p,R)$ be the function that gives the rank-$R$ right singular subspace $\Vc_R(\Xv) = {\rm span}(\vv_1(\Xv),\ldots, \vv_R(\Xv))$ of input matrix $\Xv$. The function ${\Uc}_R$ for the left singular subspace is similarly defined. 

\begin{prop}\label{thm:bd_svd} 
Let $\Xv$ be any $n \times p$ real matrix. The following holds for any $R = 1,\ldots, n \land  p$.
    \begin{enumerate}[leftmargin = 5mm, label=(\roman*)]
        \item[(i)] For some $k \le R+1$, $(k,1) \in \BP(\Uc_R;\Xv)$. In particular, $\bp_{\row}(\Uc_R;\Xv) \le  R+1$, and $\bp_{\col}(\Uc_R;\Xv) =1$.
        \item[(ii)] For some $l \le R+1$, $(1,l) \in \BP(\Vc_R;\Xv)$. In particular, $\bp_{\row}(\Vc_R;\Xv) =1$, and $\bp_{\col}(\Vc_R;\Xv) \le R+1$. 
  \end{enumerate}
\end{prop} 
In particular, the first right singular vector $\vv_1(\Xv)$ breaks down even with contamination of one row of $\Xv$, or two elements in a block of size $(2,1)$. This shows that the standard SVD is highly sensitive to outliers in the matrix. 
Surprisingly, a family of seemingly robust algorithms for SVD approximation, which we call Element-wise Loss SVD (or ELSVD for short), turns out to have very low breakdown points. In particular, the ELSVD, studied in \cite{gabriel1979lower,liu2003robust,ke2005robust,zhang2013robust}, is given by solving the following problem: 
For $r = 1,\ldots,R$, 
\begin{align}\label{eq:robSVD}
\begin{split}
    ({d}_r^{\rho},{\uv}_r^{\rho},{\vv}_r^{\rho}) = &\argmin_{d \geq 0, \uv \in S^{n-1}, \vv\in S^{p-1}}
    \rho(\Xv - d \uv \vv^T ) \\ &+ \Pc_1(\uv) + \Pc_2(\vv),
\end{split}
\end{align}
subject to an orthogonality constraint. In \eqref{eq:robSVD}, the function $\rho( \cdot )$ is defined as $\rho(\Zv) = \sum_{i=1}^p\sum_{j=1}^n \rho_{ij}(z_{ij})$ for $\Zv=(z_{ij})$, where $\rho_{ij}:\Rb \rightarrow \Rb$ are symmetric and non-negative loss functions, and $\Pc_1$, $\Pc_2$ are regularization terms. Note that for $\rho_{ij}(z) = z^2$ without the regularization terms, the solution to \eqref{eq:robSVD} coincides with the standard SVD. 
The element-wise losses can be set as the $L_1$ or Huber loss, which are commonly believed to induce robustness. Let ${\Uc}_R^{\rho}:\Rb^{n\times p} \rightarrow {\gr}(n,R)$ and ${\Vc}_R^{\rho}:\Rb^{n\times p} \rightarrow {\gr}(p,R)$ be the functions that provide the $R$th left and right singular spaces, respectively, obtained from ELSVD using some $\rho$, $\Pc_1$, and $\Pc_2$.

\begin{thm}\label{thm:bd_rho}
Suppose that  $\rho_{ij}$ satisfies that $\rho_{ij}(z)\rightarrow \infty$ as $|z| \rightarrow \infty$ (for all $i,j$), and the functions $\Pc_1$, $\Pc_2$ restricted to the domains $\uv\in S^{n-1}$ and $\vv \in S^{p-1}$ respectively are each upper bounded. Then, the conclusions of Proposition~\ref{thm:bd_svd} hold when $\Uc_R$ and $\Vc_R$ are each replaced by $\Uc_R^{\rho}$ and  $\Vc_R^{\rho}$.
\end{thm}

In contrast, the SVD approximations given by the proposed SpSVD have higher breakdown points, as we explain below. Let ${\Vc}_R^{\sph}: \Rb^{n\times p} \rightarrow {\gr}(p,R)$ be given by ${\Vc}_R^{\sph}(\Xv) = {\rm span}(\vv_1^\sph, \ldots, \vv_R^\sph)$, the dimension-$R$  right singular subspace approximated by SpSVD. The left singular subspace function ${\Uc}_R^{\sph}$ is defined similarly. 

Recall that $\widetilde{\Xv}_\row = (\xv_1/\|\xv_1\|_2, \dots, \xv_n/\|\xv_n\|_2)^T \in \Rb^{n \times p}$ denotes the row-normalized matrix. Let $\Piv_R \coloneqq \Piv_R(\Xv) $ be  the $p\times p$ matrix of projection onto the orthogonal complement of ${\Vc}_R^{\sph}(\Xv)$, and let $\lambda_r(\Xv)$ denotes the $r$th largest singular value of $\Xv$.
Define 
\begin{equation}\label{eq:nr}
\begin{split}
    n_R \coloneqq n_R(\Xv) = \min_{1 \leq k\leq n} \Big\{&k : k \geq \inf_{\widetilde{\Xv}_k}\big\{\big(\lambda_R(\widetilde{\Xv}_k)\big)^2 \\
    & - \big(\lambda_1(\widetilde{\Xv}_k \Piv_R)\big)^2 \big\}\Big\},
\end{split}
\end{equation}  
where infimum is taken over all possible submatrices $\widetilde{\Xv}_k \in \Rb^{(n-k)\times p}$ of $\widetilde{\Xv}_\row$ obtained by choosing $n-k$ rows of $\widetilde{\Xv}_\row$. Similarly, let $p_R \coloneqq p_R(\Xv)$ be given by $p_R(\Xv) = n_R(\Xv^T)$. The number $n_R$ represents a lower bound on the minimum number of rows that can break down ${\Vc}_R^{\sph}$, since it can be shown that 
  the inequality in \eqref{eq:nr} is a necessary condition for ${\Vc}_R^{\sph}$ to break down with $k$ outliers. Similarly, the number $p_R$ is a lower bound for the column-wise breakdown point for the left singular subspace approximation $\Uc_R^{\sph}$.
  
\begin{thm}\label{thm:bd_sph} Let $\Xv$ be any $n \times p$ real matrix. The following holds for any $R = 1,\ldots, n \land  p$.
    \begin{enumerate}[leftmargin = 5mm, label=(\roman*)]
\item $\bp_{\row}(\Uc_R^{\sph};\Xv) \leq R+1$, $\bp_{\col}(\Uc_R^{\sph};\Xv) \geq p_R$, and for any $(k,l) \in \BP(\Uc_R^{\sph};\Xv)$, $(1,p_R) \preceq (k,l) $.
        \item $\bp_{\row}(\Vc_R^{\sph};\Xv) \geq n_R$, $\bp_{\col}(\Vc_R^{\sph};\Xv) \leq R+1$, and for any $(k,l) \in \BP(\Vc_R^{\sph};\Xv)$, $(n_R,1) \preceq (k,l) $
    \end{enumerate}
\end{thm}

In Theorem~\ref{thm:bd_sph}, $(n_R,1)$ represents a lower bound of $\BP(\Vc_R^{\sph};\Xv)$. This bound may be perceived as too low because it only includes the column size $1$. However, the theorem also implies that no breakdown occurs for $\Vc_R^{\sph}$ by  contamination of blocks of size $(k,l)$, for any $k < n_R$ and $l=1,\dots,p$. This is because every point in $\BP(\Vc_R^{\sph};\Xv)$ is greater than or equal to $(n_R,1)$. In other words, if $(k,k) \in \BP(\Vc_R^{\sph};\Xv)$, then $k \geq n_R$.

Evaluating the lower bounds $n_R$ and $p_R$ appears to be challenging. Through numerical experiments, we  have observed that while $n_R$ and $p_R$ depend  on the matrix $\Xv$, they tend to be larger when there is a larger gap between $\lambda_R(\Xv)$ and $\lambda_{R+1}(\Xv)$; See Section~\ref{app:lower_bound} of the supplementary material.  

We additionally investigate the breakdown points in COP \citep{rahmani2017coherence}, which adopt screening out potential outliers; See Section~\ref{app:bdpoint_ex} of the supplementary material.

%--------------------------------------------------------------------------------------------
\section{NUMERICAL STUDIES}\label{sec:numeric}

In this section we evaluate the empirical performance of SpSVD in terms of accuracy, robustness and computational scalability, making comparisons to the standard SVD algorithm, ELSVD with Huber's loss of \cite{zhang2013robust}, RPCA of \citet{candes2011robust}, R2PCP of \citet{brahma2017reinforced}, as well as COP of \citet{rahmani2017coherence}.

RPCA  aims to recover a low-rank matrix and a sparse outlier matrix from the data matrix by decomposing it into the sum of the two matrices, while R2PCP \citep{she2016robust,brahma2017reinforced} models not only  a sparse outlier matrix but also the orthogonal complement of the low-rank matrix where outliers lying. Since both approaches provide robust low-rank approximations, we apply the standard SVD to the low-rank approximation to extract the left and right singular vectors and singular value approximations. COP adopts normalizing and filtering out outliers in constructing PC directions.

\paragraph{Simulation Experiment}   
We model the data matrix without outliers $\Xv \in \Rb^{n \times p}$ as the sum of low-rank $\Lv$ and $\Ev$ consisting of standard normal random noises with $n=200$ and $p=100$. The low-rank $\Lv$ is $\Lv= \sum_{r=1}^3 d_r \uv_r \vv_r^T$, $(d_1,d_2,d_3)=(80,70,60)$, and $\Uv = (\uv_1,\uv_2,\uv_3)$ and $\Vv= (\vv_1,\vv_2,\vv_3)$ are randomly sampled where the uniform distribution on the set of orthogonal matrices, respectively. We add a sparse outlier matrix $\Sv$ with $\|\Sv\|_{\rm{F}} =1$ multiplied by $\eta$, a scaling parameter, to the data matrix $\Xv$, i.e.,
\begin{align*}
\Xv = \Lv  + \Ev, \quad \Xv^{\eta} = \Lv  + \eta \Sv+ \Ev.
\end{align*}
The outlier matrix $\Sv$ has non-zero elements only in arbitrary blocks with the block-size $(0.05n,0.05p) \in \Bc_{n,p}$ where the block-size is set based on the computation of the lower bounds $n_R$ and $p_R$ described in Theorem~\ref{thm:bd_sph} for a reduced-size matrix; see Section~\ref{app:lower_bound} of the supplementary material. 
We evaluate the performance of different methods by gradually increasing $\eta$ from 0 to 1000, and repeating the simulation 100 times for each value of $\eta$.

\begin{figure}[ht]
\vspace{.3in}
     \centering
    \begin{subfigure}[b]{0.5\textwidth}
         \centering
         \includegraphics[width=\textwidth]{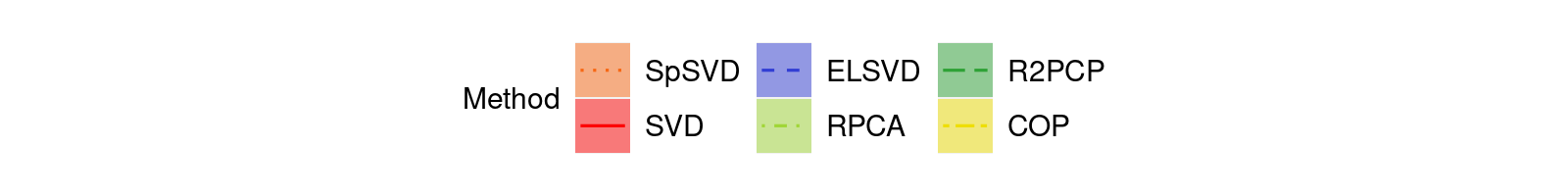}
    \end{subfigure}\\
    \begin{subfigure}[b]{0.20\textwidth}
         \centering
         \includegraphics[width=\textwidth]{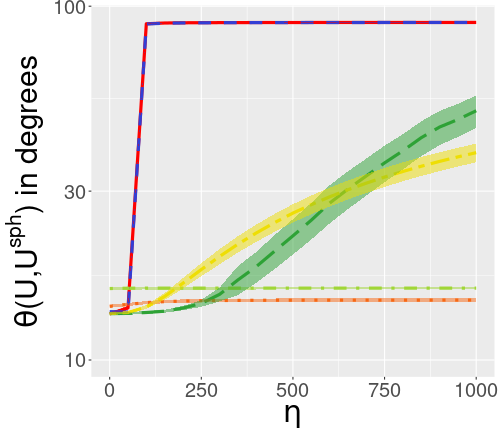}
         \caption{Singular Vector}
         \label{subfig:sim_v}
     \end{subfigure}
     \begin{subfigure}[b]{0.20\textwidth}
         \centering
         \includegraphics[width=\textwidth]{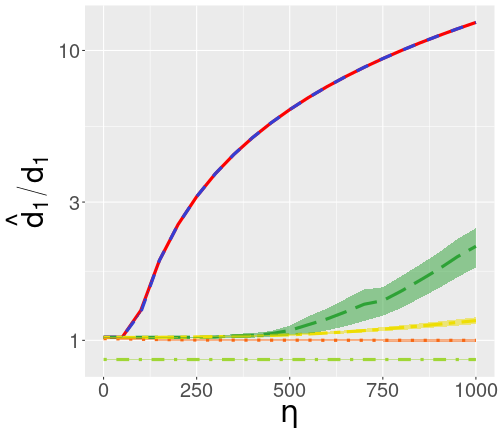}
         \caption{Singular Value}
         \label{subfig:sim_d}
     \end{subfigure}
    \begin{subfigure}[b]{0.20\textwidth}
         \centering
         \includegraphics[width=\textwidth]{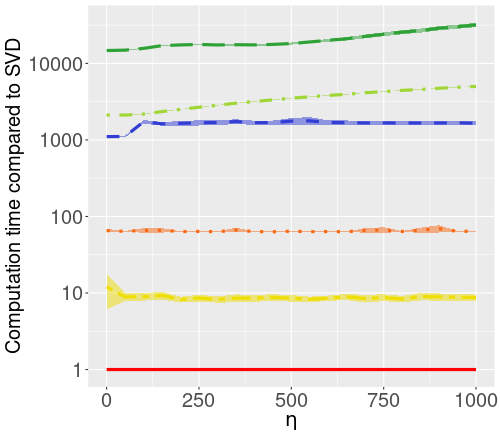}
         \caption{Time}
         \label{subfig:sim_t}
     \end{subfigure}
     \begin{subfigure}[b]{0.20\textwidth}
         \centering
         \includegraphics[width=\textwidth]{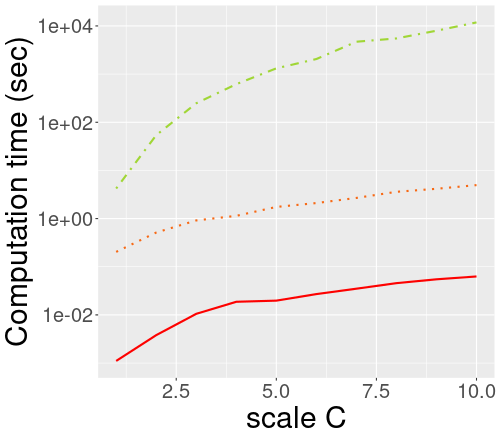}
         \caption{Scalability}
         \label{subfig:simnp_final}
    \end{subfigure}
    \vspace{.3in}
        \caption{The Approximation Accuracy,  Robustness, and Computation Times Against Increasing Magnitude $\eta$ of Outliers In (a), (b), And (c). Panel (d) shows the computation times over varying size of input matrix. The y-axes are in log-scale.} 
        \label{fig:sim}
\end{figure}

Figure~\ref{subfig:sim_v} illustrates the angle between 3-dimensional left singular subspace of $\Lv$ and the approximated  subspaces given by the six methods. We do not report the results for the right singular subspaces. Except that COP recovers well the right singular subspaces, the result is similar to that of the left singular subspace. Panel (b) shows the ratio of the approximated largest singular value of $\Xv^\eta$ to the largest singular value of $\Lv$. In terms of these values, we find that SVD, ELSVD, R2PCP, and COP do not exhibit robustness. 
While both SpSVD and RPCA appear to show desirable performances (with approximation errors for the singular subspace around $15^\circ$ across varying $\eta$), SpSVD performs consistently better for singular value approximation ($\widehat{d}_1^\sph / d_1 \approx 1$ and  $\widehat{d}_1^{\rm RPCA} / d_1 \approx 0.86$). On the other hand, SpSVD boasts remarkably faster computation times compared to robust SVD approaches; see panel (c). Specifically, R2PCP requires  approximately 500 times longer computation times on average for $\eta = 1000$.

To measure computational scalability of these algorithms, we increases the size of the matrix from $200 \times 100$ to $2000 \times 1000$, as shown in Figure~\ref{subfig:simnp_final}. For this comparison, we have only compared SpSVD with the standard SVD as a baseline method and RPCA, the best-performing method among the existing methods. While the computation time for SpSVD increase at a rate similar to that of SVD, the computation time for RPCA increases much faster. 
Overall, we find that SpSVD not only accurately approximates the SVD in the presense of massive outliers, but also is fast, requiring only 70 times longer computation times than the standard SVD, in contrast to 
the other methods which require more than 1000 times longer times for large-scale data. We have also experimented with a higher rank case; see Section~\ref{app:rank9} of the supplementary material.

\paragraph{Experiment on Gene Expression Data Matrix}   
Using the data consisting of gene expression levels of  $n=168$ patients with small invasive ductal carcinomas, obtained from a comparative genomic hybridization array \citep{gravier2010prognostic}, we consider a scenario in which a block of size (16, 16) is contaminated (perhaps by a physical contamination of the array).  See Section~\ref{app:real} of the supplementary material for a detailed description of data pre-processing, and the mechanism of contamination.

The rank-$2$ approximation $\Xv_2$ of the original, uncontaminated data $\Xv$ (via the standard SVD) is considered as a groud truth. 
We also obtain three rank-$2$ approximations of the contaminated data: by the standard SVD, denoted by $\widehat{\Xv}_2^{\rm{svd}}$ using the standard SVD, 
$\widehat{\Xv}_2^\sph$ and $\widehat{\Xv}_2^{\rm{rpca}}$ obtained by SpSVD and RPCA, respectively. 
As expected, the SVD approximation is heavily affected by the outliers. In contrast, $\widehat{\Xv}_2^\sph$ and $\widehat{\Xv}_2^{\rm{rpca}}$   provide accurate and highly robust approximations of $\Xv_2$. These are graphically depicted in Figure~\ref{fig:real} in Section~\ref{app:real} of the supplementary material.

To further compare the performances of SpSVD and RPCA, we repeat the above experiment for 100 times, each with different realizations of random contamination. The quality of approximation $\widehat{\Xv}_2^\sph$ is measured by the relative error, $r(\widehat{\Xv}_2^\sph) \coloneqq \|\Xv - \widehat{\Xv}_2^\sph\|_{\rm{F}}/\|\Xv - \Xv_2\|_{\rm{F}}$. The errors $r(\widehat{\Xv}_2^{\sph})$ and $r(\widehat{\Xv}_2^{\rm{rpca}})$ are found to be similar, with values of 1.02 and 1.01, on average, respectively, and are much smaller than  $r(\widehat{\Xv}_2^{\rm{svd}}) \approx 51.56$. On average, SpSVD takes only 0.13 seconds for rank-2 approximation, while RPCA takes 68.56 seconds for the same task, which is more than 500 times larger than SpSVD's computation times.  

%--------------------------------------------------------------------------------------------
\section{CONCLUSIONS AND FUTURE WORKS}
In this paper, we have presented the SpSVD algorithm that provides a highly scalable, accurate, and robust approximation of SVD. To demonstrate  robustness, we have extended the classical notion of breakdown point to incorporate breakdown of unit vectors and subspaces, with respect to row-wise, column-wise, and block-wise contamination of a data matrix. Using the novel notion of block-wise breakdown points, 
our theoretical analysis further validates that our approach produces robust singular vector approximations  in the presence of block-wise contamination of input matrix contamination, outperforming the classical SVD and ELSVD. Through numerical studies, we have  demonstrated not only the desirable accuracy and robustness of our approach, but also a much higher computation efficiency compared to existing methods. 
 
We point out that our theoretical analysis can be further improved. In particular, in Theorem~\ref{thm:bd_sph} we have only provided lower bounds of breakdown points, and these bounds are by no means optimal. Moreover, evaluating the lower bounds  for a given $\Xv$ seems very challenging, as computing $n_R$ involves $\sum_{k=1}^{n_R}\frac{n!}{(n-k)!k!}$ comparisons,   which becomes impractical for large-scale data with large values of $n$. A potential future direction of research is to explore alternative approaches that provide a tightened lower bound with improved computational efficiency.

Additionally, we have not theoretically examined the block-wise breakdown points for many other robust SVD methods, including RPCA of \citet{candes2011robust}, and R2PCP of \citet{brahma2017reinforced}. The breakdown of singular subspaces, described in terms of the maximal deviation from uncontaminated subspace, may be explored for  other robust SVD methods, but appears to be technically challenging. Finally, we have not addressed the issue of rank selection, which becomes even harder with contaminated data. We leave this as a future research topic. 

\subsubsection*{Acknowledgements}
This work was supported by Samsung Science and Technology Foundation under Project Number SSTF-BA2002-03.

% \bibliographystyle{asa}
% \bibliography{library}
%------------------------------------------------------------------------

%---------------------------------------------------------------------------
\section{Appendix}
\subsection{TECHNICAL DETAILS}\label{app:proof}
\subsubsection{Statistical Accuracy Theorem}\label{subapp:consistency}
A statistical accuracy of $\vv_r^\sph$ can be measured asymptotically. For this purpose, assume that each $p$-vector $\xv_i$ is independently sampled from a mean-zero elliptical distribution \citep{cambanis1981theory}, $\mathcal{F}_{\Sigmav}$ with covariance matrix $\Sigmav = \Vv\Lambdav\Vv^T = \sum_{i=1}^p \lambda_i\vv_i\vv_i^T$ (satisfying $\lambda_i > \lambda_{i+1}$). Translating a result of \citet{locantore1999robust_discuss} on the robust PCA estimation of \cite{locantore1999robust}, we observe the following: 

\begin{theorem*}[\citet{locantore1999robust_discuss}]{\label{thm:consistency}}
  Let $\Xv_n = [\xv_1,\ldots,\xv_n]^T$ for increasing $n$, where $\xv_i$'s are independently sampled from $\mathcal{F}_{\Sigmav}$. If $\Lambdav = {\rm diag}(\lambda_1,\ldots,\lambda_p)$ consists of distinct non-negative diagonal elements,  then for $r =1,\ldots, p$, $\vv_r^\sph(\Xv_n)$ is a consistent estimator of the $r$th PC direction $\vv_r$ in the sense that 
   $\|\vv_r - \vv_r^\sph(\Xv_n) \|_2 \rightarrow 0$ almost surely as $n \rightarrow \infty$.
\end{theorem*} 
Note that the conclusion of above theorem holds when $\vv_r^\sph(\Xv_n)$ is replaced by the right singular vector $\vv_r(\Xv_n)$, which may be used to justify the use of SVD in the estimation of PC directions. The theorem implies that $\vv_r^\sph(\Xv)$ is a statistically accurate approximation of the population PC direction. 

\begin{proof}
    It is known that, for independent and identically distributed (\rm{i.i.d.}) random variables with a finite covariance matrix, the sample covariance matrix almost surely converges to the population covariance matrix. Furthermore, when the eigenvalues of the population covariance matrix are distinct, the eigenvector corresponding to the $j$th largest eigenvalue of the sample covariance matrix converges to that of the population covariance matrix, almost surely.

    It is enough to show that $\Var(\frac{\xv_1}{\|\xv_1\|_2}) = \Vv\Tilde{\Lambdav}\Vv^T$  where $\Vv = (\vv_1,\dots,\vv_p)$ is a matrix consisting of PC directions, and $\Tilde{\Lambdav}=\diag(\Tilde{\lambda}_1,\dots,\Tilde{\lambda}_p)$ with $\Tilde{\lambda}_1 > \dots >\Tilde{\lambda}_p$.
   Denote $\Zv = \Vv^T \xv_1$. Note that the $\xv_1$ is from an elliptical distribution, so that characteristic function of $\xv_1$ is represented by $\phi_{\xv_1}(\tv) = \psi (\tv^T \Vv\Lambdav\Vv  \tv)$ for some scalar function $\psi: \Rb \rightarrow \Rb$, and the characteristic function of $\Zv$ is represented by $\phi_{\Zv}(\tv) = \psi (\tv^T \Lambdav  \tv)$. It implies that $pdf_{\Zv} (z_1,\dots,z_j,\dots,z_p) = pdf_{\Zv} (z_1,\dots,-z_j,\dots,z_p)$ for any $j=1,\dots,p$.
    Thus, we have
    \begin{align*}
        \E(\frac{Z_i}{\|\Zv\|_2}) =0 ,\quad \E(\frac{Z_i Z_j}{\|\Zv\|_2^2}) = 0.
    \end{align*}
    It implies that the off diagonal elements of $\Var(\frac{\Zv}{\|\Zv\|_2})$ are zero. On the other, for $j=1,\dots,p-1,$
    \begin{align*}
        \E(\frac{Z_j^2}{\|\Zv\|_2^2}) 
        &= \E(\frac{Z_j^2}{Z_j^2 + Z_{j+1}^2 + \sum_{i\neq j, j+1} Z_i^2}) \\
        &= \E(\frac{\lambda_jY_j^2}{\lambda_j Y_j^2 + \lambda_{j+1} Y_{j+1}^2 + \sum_{i\neq j, j+1} \lambda_i Y_i^2}) \\
        &> \E(\frac{\lambda_{j+1}Y_{j+1}^2}{\lambda_{j+1} Y_{j+1}^2 + \lambda_{j} Y_{j}^2 + \sum_{i\neq j, j+1} \lambda_i Y_i^2})
        = \E(\frac{Z_{j+1}^2}{\|\Zv\|_2^2}) 
    \end{align*}
    where $\Yv = (Y_1,\dots,Y_p)^T$ is a random vector which has a characteristic function $\phi_{\Yv}(\tv) = \psi (\tv^T\tv)$.
    Thus, the diagonal elements of $\Var(\frac{\Zv}{\|\Zv\|_2})$ are decreasing. Hence,
    \begin{align*}
        \Var(\frac{\xv_1}{\|\xv_1\|_2}) = \Vv \Var (\frac{\Zv}{\|\Zv\|_2})\Vv^T = \Vv \Tilde{\Lambdav}\Vv^T
    \end{align*}
    where $\Tilde{\Lambdav}=\diag(\Tilde{\lambda}_1,\dots,\Tilde{\lambda}_p)$ with $\Tilde{\lambda}_j = \E(\frac{Z_j^2}{\|\Zv\|_2^2})$.
\end{proof}

\subsubsection{Proof of Theorem~\ref{thm:finsamp}}
\begin{proof}
    Let $\Sigmav \sb{0} = \sum \sb{j=1}^p d \sb{j} \mathbf{v} \sb{j} \mathbf{v} \sb{j}^T$ be the covariance matrix of $\frac{\mathbf{x} \sb{1}}{\lVert \mathbf{x} \sb{1} \rVert \sb{2}}$ by Statistical Accuracy Theorem. Let $\widehat{\Sigmav}_\epsilon$ be the sample covariance made with $n$ normalized samples $[\frac{\xv_{1}}{\lVert \xv_{1}\rVert} , \dots, \frac{\xv_{(1-\epsilon)n}}{\lVert \xv_{(1-\epsilon)n}\rVert},\frac{\yv_{1}}{\lVert \yv_1\rVert} , \dots, \frac{\yv_{\epsilon n}}{\lVert\yv_{\epsilon n}\rVert}]$. By Davis-Kahan theorem and taking expectation the both sides, we have 
    $$
    \mathbb{E} \left[\min\{\lVert\hat{\mathbf{v}}_{j}^\sph - \mathbf{v} \sb{j} \rVert \sb{2} , \lVert\hat{\mathbf{v}}_{j}^\sph + \mathbf{v} \sb{j} \rVert \sb{2}\}\right] \leq \frac{2^{2/3}}{\delta_j} \mathbb{E} \left[\lVert \hat{\Sigmav} \sb{\epsilon} - \Sigmav \sb{0} \rVert\right]
    $$ with spectral norm $\lVert \cdot \rVert$. The right-hand side can be bounded as 
    \begin{align*}
        \mathbb{E} \left[ \lVert \hat{\Sigmav} \sb{\epsilon} - \Sigmav \sb{0} \rVert \right] \leq& \; \epsilon \mathbb{E}\left[\lVert \frac{1}{\epsilon n} \sum \sb{i=1}^{\epsilon n} \frac{\mathbf{y} \sb{i} \mathbf{y} \sb{i}^T}{\lVert \mathbf{y} \sb{i} \rVert \sb{2}^2} \rVert\right]\\
        &+ (1-\epsilon) \mathbb{E}\left[\lVert \frac{1}{(1-\epsilon) n} \sum \sb{i=1}^{(1-\epsilon ) n} \frac{\mathbf{x} \sb{i} \mathbf{x} \sb{i}^T}{\lVert \mathbf{x} \sb{i} \rVert \sb{2}^2} - \Sigmav \sb{0}  \rVert\right] .
    \end{align*}
      Using Theorem 4.7.1 (Covariance estimation in \cite{vershynin2018high}), We have
      $$\mathbb{E} \left[\lVert \hat{\Sigmav} \sb{\epsilon} - \Sigmav \sb{0} \rVert\right] \leq \epsilon + \left(1-\epsilon\right) c \left(\sqrt{\frac{p}{(1-\epsilon)n}} + \frac{p}{(1-\epsilon)n}\right),$$
      where $c$ is an absolute constant. Thus,
      \begin{align*}
          \mathbb{E} \left[\lVert\hat{\mathbf{v}}\sb{j}^\sph - \mathbf{v}\sb{j} \rVert\sb{2}\right] &\leq \frac{1}{\delta_j} \lbrace 2^{\frac{3}{2}} \epsilon + (1-\epsilon) C (\sqrt{\frac{p}{(1-\epsilon)n}} + \frac{p}{(1-\epsilon)n}) \rbrace\\
          &\leq \frac{1}{\delta_j}(C' \epsilon + C'' \sqrt{(1-\epsilon)\epsilon})
    \end{align*}
    for some absolute constants $C', C'' > 0$.
\end{proof}

\subsubsection{Proof of Lemma~\ref{lem:bd_relation}}
We provide a proof for the lemma, where Lemma~\ref{lem:bd_relation} is a direct consequence of the following lemma.
\begin{lem}\label{lem:bd_relation_app}\
    \begin{enumerate}[leftmargin = 5mm, label=(\roman*)]
        \item[(i)] $\bp_{\row}(\Vc;\Xv) =k$ if and only if $(k,l) \in \BP(\Vc;\Xv)$ for some $1\leq l \leq p$ and $(i,j) \not\in \BP(\Vc;\Xv)$ for any $(i,j) \preceq (k-1,p)$. 
        \item[(ii)] $\bp_{\col}(\Vc;\Xv) =l$ if and only if $(k,l) \in \BP(\Vc;\Xv)$ for some $1 \leq k \leq n$ and $(i,j) \not\in \BP(\Vc;\Xv)$ for any $(i,j) \preceq (n,l-1)$. 
    \end{enumerate}
\end{lem}
\begin{proof}
    Assume that $\bp_\row(\Vc ; \Xv) = k$. It implies that $\Vc$ breaks down at block-size $(k,p)$ at $\Xv$, thus there exists $(i,j) \in \BP (\Vc ; \Xv)$ such that $(i,j) \preceq (k,p)$. Since $\bp_\row(\Vc ; \Xv) > k-1$, $\Vc$ does not break down at block-size $(i,j)$ at $\Xv$ for any $ (i,j) \preceq (k-1,p)$. Thus, $(k,l) \in \BP(\Vc;\Xv)$ for some $1\leq l \leq p$ and $(i,j) \not\in \BP(\Vc;\Xv)$ for any $(i,j) \preceq (k-1,p)$.

    Assume that $(k,l) \in \BP(\Vc;\Xv)$ for some $1\leq l \leq p$ and $(i,j) \not\in \BP(\Vc;\Xv)$ for any $(i,j) \preceq (k-1,p)$. Then, we have $\bp_\row(\Vc ; \Xv) \leq k$ since $(k,l) \in \BP(\Vc;\Xv)$ for some $1\leq l \leq p$, and  $\bp_\row(\Vc ; \Xv) > k-1$ since $(i,j) \not\in \BP(\Vc;\Xv)$ for any $(i,j) \preceq (k-1,p)$.

    The proof of the second part of this lemma about the column-wise brakdown points, can be given by following the lines of the above, with $\Xv$ replaced by $\Xv^T$ and $\Vc$ replaced by $\Vc^T: \Rb^{p \times n} \rightarrow \Rb$ given by $\Vc^T(\Xv^T) = \Vc(\Xv)$. Lemma~\ref{lem:bd_relation} is proved by setting $k=1$ or $l=1$ in Lemma~\ref{lem:bd_relation_app}.
\end{proof}
\subsubsection{Proof of Proposition~\ref{thm:bd_svd}}
\begin{proof}
    The proof of Proposition~\ref{thm:bd_svd} can be given by following the lines of the proof for Theorem~\ref{thm:bd_rho}, with $\rho(\Zv) = \sum_{i,j}\rho_{ij}(z_{ij})= \sum_{i,j}z_{ij}^2$ without the regularizations.

\end{proof}
\subsubsection{Proof of Theorem~\ref{thm:bd_rho}}
\begin{proof}
    Let $({\sigma}_r^{\rho},{\uv}_r^{\rho},{\vv}_r^{\rho})$ be the $r$th solution to \eqref{eq:robSVD} for $\Xv = (x_{ij})\in \Rb^{n \times p}$. 
    Let $\Uc = \Uc_R^{\rho}(\Xv)$ and let $\Uc^{\perp}$ be the orthogonal complement of $\Uc_R^{\rho}(\Xv)$. Denote a unit vector in $\Rb^n$ with the $i$th element 1 by $\ev_i$. Since the dimension of $\Uc$ is $R<n$, there exist $\wv \coloneqq d_1\ev_{l_1} + \dots + d_{R+1}\ev_{l_{R+1}}$ with $d_1^2+\dots+d_{R+1}^2=1$ such that $\wv \in \Uc^\perp$. Assume that $l_i = i$, for notational simplicity and denote $\wv = (w_1 ,\dots,w_n)^T= (d_1,\dots,d_{R+1},0,\dots,0)^T$. Let $M$ be a constant satisfying $ \sum_{(i,j) \in I } \rho_{ij}(x_{ij}) + b < M$ where $I = \{ (i,j) : i>R+1 \text{ or } j > 1 \}$ and $b = \sup_{\uv \in S^{n-1}} \Pc_1(\uv) + \sup_{\vv \in S^{p-1}} \Pc_2(\vv)$. For arbitrary $0<\eps<1$, define $\Sc_{\eps} = \{\uv \in \Rb^n : |\uv^T\wv| > \|\uv\|_2(1-\eps^2/2)\}$ and $\Bc(c) = \{\uv = (u_1,\dots,u_n) \in \Rb^n : \sum_{i=1}^{R+1} \rho_{i1}(c d_i - u_i) + \sum_{i=R+2}^{n} \rho_{i1}(x_{i1} - u_i) < M\}$. For some $M' >0$, we can find a neighborhood of $c \wv$ with radius $M'$, given by $B(c\wv,M') \coloneqq \{c \wv + \uv \in \Rb^n : \|\uv\|_2 \leq M'\}$ satisfying $\Bc(c) \subset B(c\wv,M')$ since $\rho_{ij}$ diverges at the boundaries of $\Rb$. Here $M' >0$ is independent of $c>0$. By growing $c$, we can choose $c_0$ such that $\Bc(c_0) \subset B(c_0 \wv, M') \subset \Sc_\eps$.
    Let $\Zv = (z_{ij}) = (\zv^1,\dots,\zv^p) \in \Rb^{n \times p}$ be denoted by
    \begin{align*}
        \begin{cases}
            z_{ij} = x_{ij} & \text{ if }(i,j) \in I,\\
            z_{i1} = c_0 d_i & \text{ if }i=1,\dots,R+1.\\
        \end{cases} 
    \end{align*}
    Then, 
    {\small
    \begin{align*}
        &\min_{\av \in \Rb^p} \{\rho(\Zv - \wv \av^T) + \Pc_1(\wv) + \Pc_2 (\av/\|\av\|)\}\\
        &\leq \min_{a_1 \in \Rb} \sum_{i=1}^{R+1} \rho_{i1}(z_{i1} - d_i a_1) + \sum_{j=2}^p \min_{a_j \in \Rb} \sum_{i=1}^{R+1} \rho_{ij}(z_{ij} - d_i a_j) +\sum_{i=R+2}^{n}\sum_{j=1}^{p} \rho_{ij}(z_{ij}) + b\\
         &= \min_{a_1 \in \Rb} \sum_{i=1}^{R+1} \rho_{i1}(c_0 d_i - d_i a_1) + \sum_{j=2}^p \min_{a_j \in \Rb} \sum_{i=1}^{R+1} \rho_{ij}(x_{ij} - d_i a_j) +\sum_{i=R+2}^{n}\sum_{j=1}^{p} \rho_{ij}(x_{ij}) + b\\
         &\leq 0 + \sum_{(i,j)\in I} \rho_{ij}(x_{ij})+ b < M.
    \end{align*}
    }
    For any $\uv \not\in \Sc_\eps$ with $\|\uv\|_2 = 1$, $a \uv \not\in \Bc(c_0)$ for any $a \in \Rb$. Hence, 
    \begin{align*}
        &\min_{\av \in \Rb^p} \{ \rho(\Zv - \uv \av^T) + \Pc_1(\wv) + \Pc_2 (\av/\|\av\|)\}\\
        &\geq \min_{a_1 \in \Rb} \sum_{i=1}^{n} \rho_{i1}(z_{i1} - v_i a_1) \geq M.
    \end{align*}
    Thus, $\wv_1^\rho \in \Sc_\eps$ where $({\etav}_r^{\rho},{\wv}_r^{\rho},{\sv}_r^{\rho})$ is the $r$th solution to \eqref{eq:robSVD} for $\Zv$. Hence we have $|(\wv_1^{\rho})^T\uv_r^{\rho}| = |(\wv_1^{\rho} - \wv + \wv)^T\uv_r^{\rho}| \leq \sqrt{2-2(1-\eps^2/2)} = \eps$ for all $r$. Let $\Uv^{\rho} = (\uv_1^\rho,\dots,\uv_R^\rho) \in \Rb^{n \times R}$ be a matrix whose columns are basis vectors of $\Uc$, and $\Wv^\rho = (\wv_1^\rho,\dots,\wv_R^\rho) \in \Rb^{n \times R}$ be a matrix whose columns are basis vectors of $\Wc \coloneqq \Uc_R^\rho(\Zv)$. Let $\Av = \Wv^T\Vv \Vv^T \Wv$ and $\Tilde{\Av} = (\Tilde{a}_{ij})$ with the $(i,j)$th element $\Tilde{a}_{ij} = 0$ if $i=1$ or $j=1$, and $\Tilde{a}_{ij} = a_ij$ otherwise. By Weyl's Theorem,
    \begin{align*}
        \cos^2(\theta(\Uc,\Wc)) \leq \|\Av - \Tilde{\Av}\|_{\rm{F}} \leq 2R^2\eps.
    \end{align*}
    Since $\eps$ is arbitrary, $\sup_{\Zv_{R+1}^1} \theta(\Uc_R(\Xv),\Uc_R(\Zv_{R+1}^1)) =\frac{\pi}{2}$. Thus, $(k,1) \in \BP(\Uc_R^\rho ; \Xv)$ for some $k \leq R+1$. Moreover, it implies that, $\bp_{\col}(\Uc_R^{\rho};\Xv)=1$ and $\bp_{\row}(\Uc_R^{\rho};\Xv)\leq R+1$.

    The proof of the second part of this theorem can be given by following the lines of the above, with $\Xv$ replaced by $\Xv^T$ and $\Uc_R^\rho$ replaced by $\Uc_R^{\rho T}: \Rb^{p \times n} \rightarrow \Rb$ given by $\Uc_R^{\rho T}(\Xv^T) = \Vc_R^{\rho}(\Xv)$.
\end{proof}
\subsubsection{Proof of Theorem~\ref{thm:bd_sph}}
\begin{proof}
    We prove the second part of this theorem about the $R$th right singular space, $\Vc_R^{\sph}$. Assume that $\bp_\row(\Vc_R^{\sph};\Xv) = m < n_R$. For an arbitrary small $\eps$, take a corrupted data $\Zv_m = (\zv_1,\dots,\zv_n)$ such that $|I_0| = |\{i : \zv_i = \xv_i\}| = n-m$, and choose $\vv_{\perp} \in \Vc_R^{\sph}(\Zv_m)$ such that $\|\Piv_R^\perp \vv_{\perp}\|_2 < \eps$ where $\Piv_R^\perp$ is the projection matrix of $\Vc_R^{\sph}(\Xv)$. Let $\Piv_R = (\Iv - \Piv_R^\perp)$ be the projection matrix of the orthogonal complement of $\Vc_R^{\sph}(\Xv)$. Then,
    \begin{align*}
        \vv_{\perp}^T(\sum_{i =1}^n \frac{\zv_i \zv_i^T}{\zv_i^T\zv_i})\vv_{\perp} 
        &= \vv_{\perp}^T\big(\sum_{i \not\in I_0} \frac{\zv_i \zv_i^T}{\zv_i^T\zv_i}\big)\vv_{\perp} + \vv_{\perp}^T\big(\sum_{i \in I_0}\frac{\xv_i \xv_i^T}{\xv_i^T\xv_i}\big)\vv_{\perp}\\
        &\leq m + \vv_{\perp}^T\big(\sum_{i \in I_0} (\Piv_R + \Piv_R^\perp) \frac{\xv_i \xv_i^T}{\xv_i^T\xv_i} (\Piv_R + \Piv_R^\perp) \big)\vv_{\perp}\\
        &\leq m + \vv_{\perp}^T\big(\sum_{i \in I_0}\Piv_R \frac{\xv_i\xv_i^T}{\xv_i^T\xv_i}\Piv_R \big) \vv_{\perp} + 3(n-m)\eps\\
        &\leq m + \big(\lambda_1(\Tilde{\Xv}_m\Piv_R )\big)^2 + 3(n-m)\eps
    \end{align*}
    where $\Tilde{\Xv}_m \in \Rb^{(n-m)\times p}$ is the submatrix of $\Tilde{\Xv}_\row$ obtained by choosing $(n-m)$ rows indexed by $I_0$.
    On the other hand, 
        \begin{align*}
        \vv_{\perp}^T(\sum_{i =1}^n \frac{\zv_i \zv_i^T}{\zv_i^T\zv_i})\vv_{\perp} &\geq \sup_{\Vc \in \gr (p,R)}\min_{\vv \in \Vc}\{\vv^T(\sum_{i =1}^n \frac{\zv_i \zv_i^T}{\zv_i^T\zv_i})\vv\}\\
        &= \sup_{\Vc \in \gr (p,R)}\min_{\vv \in \Vc}\{\vv^T(\sum_{i \not\in I_0} \frac{\zv_i\zv_i^T}{\zv_i^T\zv_i})\vv + \vv^T(\sum_{i \in I_0}\frac{\xv_i \xv_i^T}{\xv_i^T\xv_i})\vv\}\\
        &\geq \big(\lambda_R(\Tilde{\Xv}_m)\big)^2.
    \end{align*}
    It implies that 
    \begin{align*}
        (1-3\eps)m \geq  \big(\lambda_R(\Tilde{\Xv}_m)\big)^2 - \big(\lambda_1(\Tilde{\Xv}_m\Piv_R )\big)^2 - 3n\eps,
    \end{align*}
    thus, 
    \begin{align*}
        (1-3\eps)m \geq \inf_{\Tilde{\Xv}_m}\{\big(\lambda_R(\Tilde{\Xv}_m)\big)^2 - \big(\lambda_1(\Tilde{\Xv}_m\Piv_R )\big)^2\} - 3n\eps.
    \end{align*}
    Since $\eps$ is arbitrary, 
    \begin{align*}
        m \geq \inf_{\Tilde{\Xv}_m}\{\big(\lambda_R(\Tilde{\Xv}_m)\big)^2 - \big(\lambda_1(\Tilde{\Xv}_m\Piv_R )\big)^2\}.
    \end{align*}
    It is a contradiction since $m <n_R$. Thus, $\bp_\row(\Vc_R^{\sph};\Xv) \geq n_R$. It implies that $(n_R,1) \preceq (i,j)$ for any $(i,j) \in \BP(\Vc_R^{\sph};\Xv)$.

    To show that $\bp_\col(\Vc_R^{\sph};\Xv) \leq R+1$, let $\Vc = \Vc_R^\sph(\Xv)$ and let $\Vc^\perp$ be the orthogonal complement of $\Vc_R^{\sph}(\Xv)$. Denote the unit vector in $\Rb^p$ with the $j$th element 1 by $\ev^j$. Since the dimension of $\Vc$ is $R<p$, there exist $\sv \coloneqq d_1\ev^{l_1} + \dots + d_{R+1}\ev^{l_{R+1}}$ with $d_1^2+\dots+d_{R+1}^2=1$ such that $\sv \in \Vc^\perp$. Assume that $l_j = j$, for notational simplicity and denote $\sv = (w_1 ,\dots,w_p)^T = (d_1,\dots,d_{R+1},0,\dots,0)^T$. For $\Xv = (\xv^1,\dots,\xv^p) \in \Rb^{n \times p}$ and $c > 0$, define a corrupted matrix constructed by replacing $R+1$ columns of $\Xv$, by $\Zv^{R+1}(c) = (\zv_1(c),\dots,\zv_n(c))^T = (\zv^1(c),\dots,\zv^{p}(c))$ with $\zv^j(c) = c d_j \1v_n $ for $j=1,\dots,R+1$, and $\zv^j(c) = \xv^j$ for $j=R+2,\dots,p$. Let $\Yv(c) = (\yv_1(c),\dots,\yv_n(c))^T \in \Rb^{n \times p}$ be a row-normalized matrix of $\Zv^{R+1}(c)$ whose $i$th row is $\yv_i(c) = \frac{\zv_i(c)}{\|\zv_i(c)\|_2}$. Denote $\Yv \coloneqq \lim_{c \rightarrow \infty} \Yv(c) = (\sv,\dots,\sv)^T \in \Rb^{n \times p}$. By Davis-Khan Theorem, 
    \begin{align*}
        \sin \Big(\theta \big(\sv,\vv_1(\Yv(c)) \big) \Big)  \leq \frac{2}{n} \|\Yv^T\Yv - \Yv(c)^T\Yv(c)\|_{\rm{F}} \rightarrow 0
    \end{align*}
    as $c \rightarrow \infty$ where $\vv_1(\Yv(c))$ is the right singular vector corresponding to the largest singular value of $\Yv(c)$. Note that the right singular space obtained from SpSVD, $\Vc_R^\sph(\Zv(c))$, equals the right singular space, $\Vc_R(\Yv(c))$. Thus
    \begin{align*}
        \sup_{\Zv^{R+1}} \theta(\Vc_R^\sph(\Xv),\Vc_R^\sph(\Zv^{R+1}))
        &\geq \lim_{c \rightarrow \infty} \theta(\Vc_R^\sph(\Xv),\Vc_R^\sph(\Zv^{R+1}(c)) \\
        &= \lim_{c \rightarrow \infty} \theta(\Vc_R^\sph(\Xv),\Vc_R(\Yv(c))) = \frac{\pi}{2}.
    \end{align*}
    It implies that $\bp_{\col}(\Vc_R^{\sph};\Xv) \leq R+1$.

    The proof of the first part of this theorem can be given by following the lines of the above, with $\Xv$ replaced by $\Xv^T$ and $\Vc_R^\sph$ replaced by $(\Vc_R^{\sph})^T: \Rb^{p \times n} \rightarrow \Rb$ given by $(\Vc_R^{\sph})^T(\Xv^T) = \Uc_R^{\sph}(\Xv)$.
\end{proof}
%--------------------------------------------------------------------------------------------
\subsection{COMPUTATION COMPLEXITY}
\subsubsection{Computation Complexity for Algorithm~\ref{alg:approx}}
For inputs, small rank $R$ and a data matrix $\mathbf{X} \in \mathbb{R}^{n \times p}$, Algorithm~\ref{alg:approx} consists of the following steps: normalization, rank-$R$ SVD, and finding the solution to \eqref{eq:spsvd}. Each row of $\mathbf{X}$, which is a $p$-vector, is normalized for $i=1,\dots,n$. This step requires computation time of $O(np)$. Similarly, column normalization is performed with the same complexity. We use a partial SVD algorithm to find the rank-$R$ approximation of the normalized data, which requires $O(npR)$ complexity. Finally, the solution to \eqref{eq:spsvd} is computed using a selection algorithm that finds the smallest value among the $npR^2$ elements. Each step for $r=1,\dots,R$ takes $npR^2$ computations, resulting in a total computation burden of $O(npR^3)$ for $k=1,\dots, R$.

\subsubsection{Computation Complexity of Other Methods}
In Section~\ref{sec:numeric}, we performed comparison of SpSVD with other methods. \citet{brahma2017reinforced} presents computational complexity of part of the R2PCP algorithm in their paper. Although they propose new batch version of algorithm to get reduced computational complexity, they did not reveal the total computational complexity of the entire process. On the other hand, \citet{candes2011robust} claims that the dominant cost in RPCA algorithm comes from computing one partial SVD per iteration. Although the number of iterations in their algorithm appears to remain nearly constant regardless of dimension, the algorithm in the \texttt{R} package \texttt{rpca} \citep{maciek_rpca} requires significantly more time to run compared to other methods.

%---------------------------------------------------------------------------------------------------
\subsection{ADDITIONAL DETAILES IN NUMERICAL STUDIES}
In this numerical studies, we generate the data, and implement standard SVD and SpSVD using \texttt{R}. We used the \texttt{R} package \texttt{rpca} \citep{maciek_rpca} for RPCA, \texttt{Matlab} implementation for COP, and R2PCP \citep{brahma2017reinforced} was implemented using their source code in \texttt{Matlab}.

\subsubsection{Data Generation and Adding Contamination in the Simulation Experiment}
To generate randomly sampled matrices $\mathbf{U}$ and $\mathbf{V}$ from the uniform distribution on the set of orthogonal matrices, respectively, we create a matrix consisting of standard normal random noise and apply QR decomposition. It is well known that this orthogonal matrix obtained from QR decomposition follows the uniform distribution on the set of orthogonal matrices of the same size.

To generate the sparse outlier matrix $\Sv$, we arbitrary choose $0.05\%$ of the indices as $I=\{i_1,\dots,i_{10}\} \subset \{1,\dots,n\}$ and $J = \{j_1,\dots,j_5\} \subset \{1,\dots,p\}$, respectively. Note that $|I| \geq 4$ and $|J| \geq 4$. Let $\Uv_I \in \Rb^{|I| \times 3}$ and $\Vv_J \in \Rb^{|J| \times 3}$ be the submatrices of $\Uv \in \Rb^{n \times 3}$ and $\Vv \in \Rb^{p \times 3}$ consisting of rows indexed by $I$ and $J$, respectively. We can find solutions $\Tilde{\av} = (\Tilde{a}_1,\dots,\Tilde{a}_{10})' \in \Rb^{|I|}$ and $\Tilde{\bv} =(\Tilde{b}_1,\dots,\Tilde{b}_5)' \in \Rb^{|J|}$ satisfying $\Uv_I^T \Tilde{\av} = \0v_{3}$ and $\Vv_J^T \Tilde{\bv} = \0v_{3}$. We define the sparse outlier matrix by $\Sv = \av \bv^T /\|\av\|_2\|\bv\|_2$, where $\av = (a_1,\dots,a_n)' \in \Rb^n$ has elements zero excepts for elements indexed by $I$ as $a_{i_l} = \Tilde{a}_l$ for $l=1,\dots,10$, and $\bv = (b_1,\dots,b_p)' \in \Rb^p $ has elements zero excepts for elements indexed by $J$ as $b_{j_k} = \Tilde{b}_k$ for $k=1,\dots,5$. Then, the row space and column space of $\Sv$ are orthogonal to those spaces of $\Lv$, respectively.

\subsubsection{Extra Result for the Simulation Experiment}\label{app:ex_res}

Due to computational limitations, we compared SVD, SpSVD, and RPCA without iterations for increasing matrix sizes in Section~\ref{sec:numeric}. In addition to those comparisons, we also provide experiments with 100 iterations while increasing the matrix size from $200 \times 100$ to $600 \times 300$. For $C=1, 1.2, \dots, 2.8, 3$, with $n=200$ and $p=100$, we generate a data matrix $\Xv \in \Rb^{C \cdot n \times C \cdot p}$. To maintain the magnitude of singular values and outliers, we set $\eta$ as $500C$ and the singular values $(d_1,d_2,d_3)$ as $C \cdot (80,70,60)$. The block of $\mathbf{S}$ containing the outliers has sizes of $(0.05nC,0.05pC)$. The results are presented in Figure~\ref{fig:simnp}.

\begin{figure}[t]
     \centering
    \begin{subfigure}[b]{0.6\textwidth}
         \centering
         \includegraphics[width=\textwidth]{fig/legend.png}
    \end{subfigure}\\
    \begin{subfigure}[b]{0.4\textwidth}
         \centering
         \includegraphics[width=\textwidth]{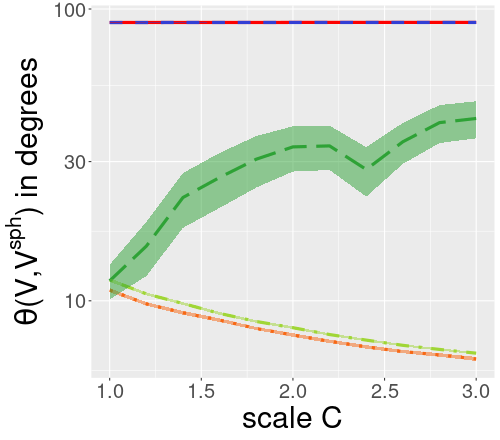}
         \caption{Right Singular Vector}
         \label{subfig:simnp_v}
     \end{subfigure}
     \begin{subfigure}[b]{0.4\textwidth}
         \centering
         \includegraphics[width=\textwidth]{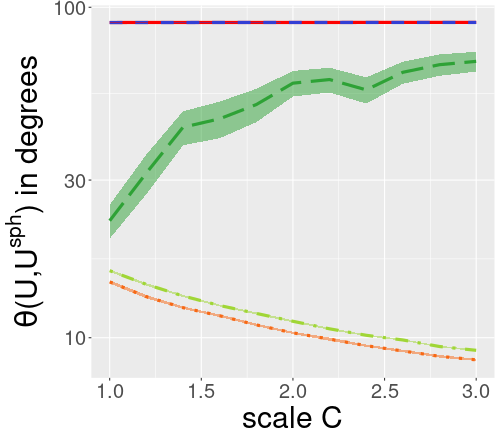}
         \caption{Left Singular Vector}
         \label{subfig:simnp_u}
     \end{subfigure}\\
    \begin{subfigure}[b]{0.4\textwidth}
         \centering
         \includegraphics[width=\textwidth]{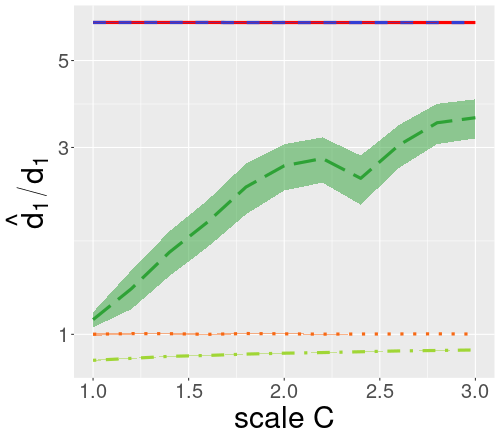}
         \caption{Singular Value}
         \label{subfig:simnp_d}
     \end{subfigure}
     \begin{subfigure}[b]{0.4\textwidth}
         \centering
         \includegraphics[width=\textwidth]{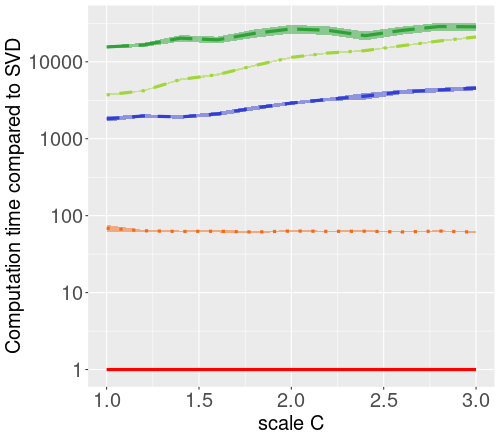}
         \caption{Computation Time}
         \label{subfig:simnp_t}
    \end{subfigure}
        \caption{The Approximation Accuracy,  Robustness, and Computation Times against Increasing Scale $C$ in (a), (b), (c) and (d).}
        \label{fig:simnp}
\end{figure}

As $C$ increases, RPCA and SpSVD exhibit better accuracy for the singular spaces and a singular value in Figure~\ref{subfig:simnp_v},\ref{subfig:simnp_u},and \ref{subfig:simnp_d}. RPCA takes longer time as $C$ increases compared to SVD, and other methods including ELSVD and R2PCP also require more time as $C$ increases compared to SVD as described in  Figure~\ref{subfig:simnp_t}.

\subsubsection{Higher Rank Case}\label{app:rank9}

The data matrix without outliers $\Xv \in \Rb^{n \times p}$ is set to be the sum of rank-9 $\Lv = (l_{ij})$ and $\Ev$ consisting of standard normal random noises with $n=1000$ and $p=500$. The rank-9 $\Lv$ is $\Lv= \sum_{r=1}^9 d_r \uv_r \vv_r^T$, $(d_1,d_2,\dots,d_8,d_9)=(750,700,\dots,400,350)$, and $\Uv = (\uv_1,\dots,\uv_9)$ and $\Vv= (\vv_1,\dots,\vv_9)$ are randomly sampled where the uniform distribution on the set of orthogonal matrices, respectively. We add a sparse outlier matrix $\Sv = (s_{ij})$ multiplied by $\eta=1000$, a scaling parameter, to the data matrix $\Xv$, i.e.,
\begin{align*}
\Xv = \Lv  + \Ev, \quad \Xv^{\eta} = \Lv  + \eta \Sv+ \Ev.
\end{align*}
The outlier matrix $\Sv$ has non-zero elements $(s_{ij})$ only in arbitrarily chosen $0.05n = 50$ rows indexed by $I_{50}$ and $0.05p = 25$ columns indexed by $J_{25}$. Thus, 
\begin{align*}
        \begin{cases}
            s_{ij} = l_{ij} & \text{ if }i \in I_{50} \text{ and } j \in J_{25},\\
            s_{ij} = 0 & \text{ otherwise }.\\
        \end{cases} 
\end{align*}
The simulation is repeated 10 times.

\begin{table}[h]
\caption{Result of The Additional Simulation Study with The Data Size of $2000$ by $1000$ and Rank 9} \label{tbl:rank9}
\begin{center}
\begin{tabular}{lrr}
\textbf{Method} &\textbf{Right angle} & \textbf{Left angle} \\
\hline \\
{SpSVD}     & 4.93 & 6.11  \\
{SVD}      & 83.47 & 81.35  \\
RPCA & 4.31  & 6.12  \\
ELSVD & 84.92 & 84.98  \\
{R2PCP}         & 77.92 & 85.49  \\
{COP}         & 4.19 & 76.33  \\
\textbf{Method} & \textbf{Ratio of singular value}  & \textbf{Time(sec)}\\
\hline \\
{SpSVD}      &0.99  & 17.39 \\
{SVD}       &61.42    &0.02 \\
RPCA  &0.96    & 2987.12 \\
ELSVD  &63.39    & 6454.69 \\
{R2PCP}          &11.01   & 890.98 \\
{COP}          &18.80    & 0.09 \\
\end{tabular}
\end{center}
\end{table}

The results are provided in Table~\ref{tbl:rank9}. Similarly to the findings in Section~\ref{sec:numeric}, SpSVD and RPCA demonstrate superior performance concerning the angles and the ratio, as indicated in the second, third, and fourth columns. 
COP, similar to our proposal (as it normalizes each vector), exhibits robustness in subspace recovery (right singular subspace). However, it's important to note that the algorithm in \cite{rahmani2017coherence} is robust PCA algorithms, not robust SVD algorithms. While they succeed in recovering the right singular vectors (the basis vectors of PC subspaces), they do not provide both left and right singular vectors along with their singular values simultaneously.

\subsubsection{Real Data}\label{app:real}
The data $\mathbf{X}$ is obtained from the raw data using the following steps. First, each variable in the raw data is transformed using the logarithm base 2 transformation. Then, we apply one-way ANOVA to the centered and normalized raw data with respect to the labels of patients, which have two categories: "good" and "poor", representing the condition of the patients. We select 500 variables from the entire gene expressions based on the p-values obtained from the one-way ANOVA. Subsequently, we obtain the data matrix $\mathbf{X}$ containing these 500 variables and scale it to have zero mean and unit variance.

For the data matrix $\Xv = (x_{ij})$, we arbitrarily choose approximately $0.1n \approx 16$ rows indexed by $I_{16}$ and $0.1n \approx 16$ columns indexed by $J_{16}$. We assign outlyingness to the data matrix by amplifying the $16^2$ elements corresponding to these rows and columns 1000 times. Thus, the data matrix with contamination, denoted as $\Xv^\eta = (x_{ij}^\eta)$, is given by
\begin{align*}
        \begin{cases}
            x_{ij}^\eta = 1000 x_{ij} & \text{ if }i \in I_{16} \text{ and } j \in J_{16},\\
            x_{ij}^\eta = x_{ij} & \text{ otherwise }.\\
        \end{cases} 
\end{align*}

\begin{figure}[t]
\vspace{.3in}
     \centering
     \begin{subfigure}[b]{0.45\textwidth}
         \centering
         \includegraphics[width=\textwidth]{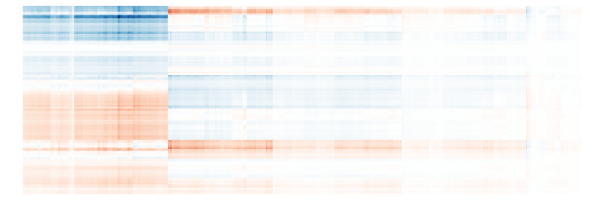}
         \caption{Ground Truth.}
         \label{subfig:real_tr}
     \end{subfigure}
     \hfill
     \begin{subfigure}[b]{0.45\textwidth}
         \centering
         \includegraphics[width=\textwidth]{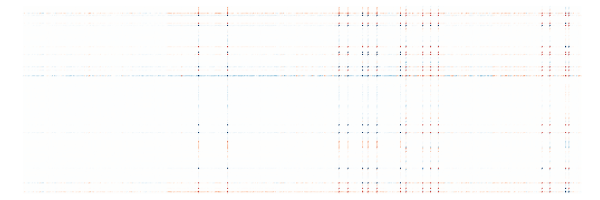}
         \caption{$\hat{\Xv}_2^{\rm{svd}}, r(\hat{\mathbf{X}}_{2}^{\rm{svd}}) =  51.56 $}
         \label{subfig:real_svd}
     \end{subfigure}
     \hfill \\
     \begin{subfigure}[b]{0.45\textwidth}
         \centering
         \includegraphics[width=\textwidth]{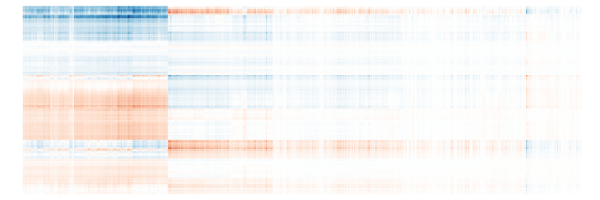}
         \caption{$\hat{\Xv}_2^{\sph}, r(\hat{\mathbf{X}}_{2}^{\sph}) = 1.02$ Time: 0.13 sec}
         \label{subfig:real_sph}
     \end{subfigure}
     \hfill 
     \begin{subfigure}[b]{0.45\textwidth}
         \centering
         \includegraphics[width=\textwidth]{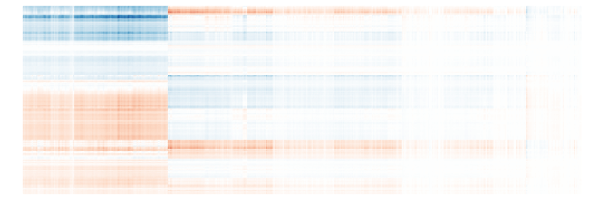}
         \caption{$\hat{\Xv}_2^{\rm{rpca}}, r(\hat{\mathbf{X}}_{2}^{\sph}) = 1.01$ Time: 68.56 sec}
         \label{subfig:real_rpca}
     \end{subfigure}
     \hfill\\
     \vspace{.3in}
        \caption{Rank-2 Approximations by SVD, SpSVD and RPCA Using Contaminated Data, Compared with the Ground Truth. }
        \label{fig:real}
\end{figure}

\subsubsection{Lower Bounds of Breakdown Points}\label{app:lower_bound}
In Section~\ref{sec:blockwise}, we presented the lower bounds $n_R$ and $p_R$ for the breakdown points of the right and left singular spaces obtained from SpSVD, respectively. To determine whether $n_R$ is equal to $k$ for $k=1,\dots,n$, we need to check the inequality in \eqref{eq:nr} for all possible submatrices of a row-normalized matrix. However, as the number of rows $n$ increases, this process becomes increasingly challenging, even for small values of $k$.

We evaluate the lower bounds of the breakdown points using small-sized data. We construct a data matrix $\Xv = \Lv + \Ev \in \Rb^{n \times p}$ with $n=200 \times 0.3 = 60$ and $p=100 \times 0.3 = 30$, where $\Lv$ is a rank-3 matrix and $\Ev$ is a matrix consisting of standard normal random noise. The rank-3 matrix $\Lv$ has singular values $(d_1, d_2, d_3)=(80,70,60)\times t$, with $t=0.3,0.45,0.6$, and the corresponding singular vectors are arbitrarily chosen.

Table~\ref{tab:lower_bound} illustrates the values of the lower bounds $n_R(\Xv)$ and $p_R(\Xv)$ for each $t=0.3, 0.45, 0.6$ and $R=3$. When larger singular values are assigned, it implies a larger gap between the two successive singular values of order $R$ and $R+1$. We observed that the lower bounds $n_R$ and $p_R$ tend to increase as the singular values increase. The lower part of Table~\ref{tab:lower_bound} also presents similar results for the case where $R=1$ and the rank-1 matrix $\Lv$ has a singular value of $60 \times t$ with $t=0.3, 0.45, 0.6$. Based on the observation where $n = 200 \times 0.3$, $p= 100 \times 0.3$, and the singular values $(80, 70, 60) \times 0.3$, depicted in the first line of Table~\ref{tab:lower_bound}, we roughly choose the size of the outlier block as $(0.05n, 0.05p)$ in the simulation experiment of Section~\ref{sec:numeric}.

\begin{table}[ht]
\caption{The Breakdown Points of the Right and Left Singular Spaces Obtained from SpSVD with Various Singular Values and $R$. Due to the computational burden, we restricted our computations to $n_R$ and $p_R$ up to 6. In cases where the lower bound exceeds or equals 7, we denote it as 7$\uparrow$, indicating that an exact value cannot be computed.}
    \label{tab:lower_bound}
    \centering
    \begin{tabular}{ p{3cm} p{2cm} p{2cm} p{2cm}   }
     \hline
     \multicolumn{4}{c}{The lower bound of breakdown points} \\
     \hline
     Singular values& $R$ &$n_R$ &$p_R$\\
     \hline
     $(80,70,60)\times 0.3 $  & 3 & 3             & 2\\
     $(80,70,60)\times 0.45$  & 3 & 5             & 2\\
     $(80,70,60)\times 0.6$   & 3 & 6             & 3\\
     \hline
     $(60)\times 0.3$         & 1 & 6             & 2\\
     $(60)\times 0.45$        & 1 & 7$\uparrow$   & 4\\
     $(60)\times 0.6$         & 1 & 7$\uparrow$   & 5\\
    \hline
    \end{tabular}
\end{table}
{
\subsection{Breakdown points in existing methods}\label{app:bdpoint_ex}

A reviewer suggested to investigate the breakdown points of the existing methods used in the empirical study. The singular vectors from the iterative algorithms of RPCA and R2PCP do not have closed-forms, and it is inherently very challenging to theoretically grasp the breakdown point for these methods. We were not able to find a right technical tool for such purpose. Note that the authors of RPCA and R2PCP did not investigate the breakdown for singular vectors. 

It turns out that COP has breakdown points upper-bounded by small numbers. (This is in contrast to our proposal, for which breakdown points are lower-bounded; recall that the higher breakdown point, the more robust a method is.) The COP algorithm, we used in Section~\ref{sec:numeric}, has a tuning parameter $m$, and consists of two steps: Screening out $m$ potential outliers, then applying the vanilla SVD. Below we provide details for rank $R=1$ SVD approximation.

Let COP$_m$ denote the COP algorithm (that removes $m$ outliers). For data matrix $\mathbf{X} \in \mathbb{R}^{n \times p}$, 
let $\mathcal{V}_1^{m}: \mathbb{R}^{n \times p} \rightarrow {\rm{Gr}(1,p)}$ be given by $\mathcal{V}_1^{m}(\mathbf{X})$, which is the one-dimensional subspace spanned by the (first) right singular vector 
of $\mathbf{X}$ obtained by COP$_m$, and $\mathcal{U}_1^{m}$ be the left singular vector, obtained by an application of COP$_m$. 
In aspects of our breakdown notions, the breakdown points of $\Vc_1^m$ and $\Uc_1^m$ for any $\Xv \in \Rb^{n \times p}$ are given as
\begin{align*}
    \bp_{\row}(\Vc_1^m;\Xv) &\leq m+1,\bp_{\row}(\Uc_1^m;\Xv) \leq 2,\\
    \bp_{\col}(\Vc_1^m;\Xv) &\leq 2,\bp_{\col}(\Uc_1^m;\Xv) = 1,\\
    (k+1,2) &\succeq (i,j) \text{ for some }(i,j) \in \BP(\Vc_1^m;\Xv),\\
    (2,1) &\succeq (i,j) \text{ for some }(i,j) \in \BP(\Uc_1^m;\Xv).
\end{align*}
Here, $(2,1) \succeq (i,j)$ for some $(i,j) \in \BP(\Uc_1^m(\Xv))$ implies that block-wise breakdown occurs by changing two elements in one column of $\Xv$. 
Our numerical experiments reflect these theoretical findings: COP$_m$ was shown to be robust and accurate in recovering the right singular subspace ($\Vc_1^m$), but failed to recover the left singular subspace ($\Uc_1^m$). Note that in these studies we have set $m$ to be the true number of outliers.

}

\end{document}